\documentclass[sigconf]{acmart}

\copyrightyear{2026}
\acmYear{2026}
\setcopyright{cc}
\setcctype{by}
\acmConference[KDD '26]{Proceedings of the 32nd ACM SIGKDD Conference on Knowledge Discovery and Data Mining V.2}{August 09--13, 2026}{Jeju Island, Republic of Korea}
\acmBooktitle{Proceedings of the 32nd ACM SIGKDD Conference on Knowledge Discovery and Data Mining V.2 (KDD '26), August 09--13, 2026, Jeju Island, Republic of Korea}
\acmDOI{10.1145/3770855.3818925}
\acmISBN{979-8-4007-2259-2/2026/08}

\newcommand{\ba}{\boldsymbol{a}}
\newcommand{\PP}{\mathbb{P}}

\usepackage{subcaption}
\usepackage{thm-restate}
\usepackage{algorithm,algorithmic}
\usepackage{dsfont}
\usepackage{amsmath}
\usepackage{mathtools}
\usepackage{amsthm}
\usepackage{multirow, enumitem}
\usepackage{booktabs} 
\usepackage[table]{xcolor} 

\theoremstyle{remark}
\newtheorem{theorem}{Theorem}

\newtheorem{lemma}[theorem]{Lemma}

\theoremstyle{definition}

\theoremstyle{remark}

\usepackage{placeins}

\begin{document}

\title{Deep Doubly Debiased Longitudinal Effect Estimation with ICE G-Computation}

\author{Wenxin Chen}
\email{wc645@cornell.edu}
\affiliation{%
  \institution{Computer Science, Cornell University}
  \city{New York City}
  \state{New York}
  \country{USA}}

\author{Weishen Pan}
\email{wep4001@med.cornell.edu}
\affiliation{%
  \institution{Weill Cornell Medicine, Cornell University}
  \city{New York City}
  \state{New York}
  \country{USA}}

\author{Kyra Gan}
\email{kyragan@cornell.edu}
\affiliation{%
  \institution{Operations Research and Information Engineering, Cornell University}
  \city{New York City}
  \state{New York}
  \country{USA}}

\author{Fei Wang}
\email{few2001@med.cornell.edu}
\affiliation{%
  \institution{Weill Cornell Medicine, Cornell University}
  \city{New York City}
  \state{New York}
  \country{USA}}


\begin{abstract}
  Estimating longitudinal treatment effects is essential for sequential decision-making but is challenging due to treatment–confounder feedback. While \emph{Iterative Conditional Expectation (ICE) G-computation} offers a principled approach, its recursive structure suffers from error propagation, corrupting the learned outcome regression models. We propose \textbf{$\mathbf{D^3}$-Net}, a framework that mitigates error propagation in ICE training and then applies a robust final correction. First, to interrupt error propagation \emph{during learning}, we train the ICE sequence using \emph{Sequential Doubly Robust} (SDR) pseudo-outcomes, which provide bias-corrected targets for each regression. Second, we employ a multi-task transformer with a \emph{covariate simulator head} for auxiliary supervision, regularizing representation learning, and a \emph{target network} to stabilize training dynamics. For the final estimate, we discard the SDR correction and instead use the uncorrected nuisance models to perform \emph{Longitudinal Targeted Minimum Loss-Based Estimation} (LTMLE) on the original outcomes. This second-stage, targeted debiasing ensures robustness and optimal finite-sample properties. Comprehensive experiments demonstrate that our model, $D^3$-Net, robustly reduces bias and variance across different horizons, counterfactuals, and time-varying confoundings, compared to existing state-of-the-art ICE-based estimators.
\end{abstract}




\maketitle
\section{Introduction}
Clinical decision-making is often sequential, with treatments adapted over time based on a patient's evolving status.
Estimating the effect of a \emph{sequence} of treatments is therefore crucial for informing optimal care.
However, this longitudinal problem is fundamentally challenging due to the treatment-confounder feedback:
time-varying confounders are affected by past treatments, and in turn affect subsequent treatment assignment and outcomes. This dynamic dependence complicates both causal identification and estimation. 

Existing studies approach this problem from two directions: treatment-model-based and outcome-model-based. \emph{Treatment-model-based} methods, such as IPTW 
\cite{rosenbaum1983central, robins2000marginal, hernan2006estimating,seedat2022continuous, melnychuk2022causal}, aim to model treatment assignment to create a pseudo-population balanced on confounders. These methods often suffer from progressively increasing variance due to the multiplicative accumulation of propensity scores. \emph{Outcome-model-based} methods, instead, try to directly model the dynamics of the outcome. 
Among these methods, \emph{Iterative Conditional Expectation (ICE)} G-computation has emerged as a prevalent approach. By breaking the problem into a series of conditional mean outcomes backward in time, ICE G-computation provides a stable alternative
\cite{robins1986new,keil2014parametric,tran2019double,petersen2014targeted}.

However, its practical efficacy hinges on accurately modeling complex, high-dimensional relationships between time-varying confounders, treatments, and outcomes—a task where traditional parametric models often fail. This limitation has driven the adoption of deep learning (DL), whose flexible function approximators (e.g., in DeepACE~\cite{frauen2023estimating}, GST-UNet~\cite{oprescu2025gst}, IGC-Net~\cite{hess2024igc}) are well-suited to these complexities. Recent advances have productively merged this flexibility with the robustness of \emph{doubly robust estimation}, as seen in Deep LTMLE~\cite{shirakawa2024longitudinal}, to hedge against model misspecification.

This DL implementation introduces a critical new challenge: autoregressive error propagation within a shared representation. Unlike traditional ICE, which fits independent regressions at each step, deep ICE trains one model whose predictions are recursively fed back as targets. Small errors in early steps compound into large biases for later ones, corrupting the training signal. This forces the model's shared feature representation to adapt to increasingly erroneous pseudo-outcomes, a form of representation drift that undermines the stability of ICE and the robustness guarantees of DR estimation in finite samples.

In this work, we address deep ICE-specific instabilities through a novel
\textbf{doubly debiased framework}, \textbf{$D^3$-Net}. The key idea is
an SDR-based objective that mitigates error propagation during
recursive nuisance estimation. We further incorporate architectural stabilization through a multi-task transformer with auxiliary supervision and a
target network, and introduce a final LTMLE targeting step for
robust inference. Our contributions are fourfold:

\begin{itemize}[itemsep=0pt, topsep=2pt, partopsep=0pt, parsep=0pt, leftmargin=*]
\item \textbf{A bias-corrected ICE G-computation objective} that incorporates SDR pseudo-outcomes to reduce error propagation during the recursive training of outcome regressions. We show that this procedure provably mitigates error propagation in Lemma~\ref{lemma:first_vs_second_order}, thereby benefiting downstream inference.
\item \textbf{A novel multi-task transformer architecture} for longitudinal data that jointly learns treatment, outcome, and covariate simulation models. It employs a target network to smooth the moving targets and a covariate simulator for auxiliary supervision, together ensuring stable and accurate nuisance estimation.
\item \textbf{A two-stage debiasing estimation procedure} that leverages the improved nuisances for LTMLE, yielding final estimates that are both robust and statistically efficient.
\item \textbf{Comprehensive empirical validation} demonstrating that our framework, $D^3$-Net, robustly achieves lower bias and variance than existing state-of-the-art longitudinal effect estimators. \footnote{Codes are publicly available at \url{https://github.com/Wenxin-Elmon-Chen/D3N}}
\end{itemize} 

\section{Related work}
Our work relates to longitudinal causal inference and off-policy evaluation (OPE) in reinforcement learning (RL), which share a core methodological challenge: estimating the value of a policy or treatment regime from observational data. Both confront analogous trade-offs among bias, variance, and stability when adjusting for time-varying confounding, yet have largely evolved in parallel.

\vspace{3pt}
\textbf{\emph{Longitudinal Causal Estimation.}}\; Methods for estimating longitudinal effects are traditionally categorized by their reliance on treatment, outcome, or both models.
Classic \textit{treatment-model–based methods} use \emph{Inverse Probability of Treatment Weighting} (IPTW) \cite{rosenbaum1983central, rosenbaum1984reducing,robins2000marginal, hernan2006estimating, cole2008constructing}, its stabilized variant \cite{xu2010use, chesnaye2022introduction}, or matching \cite{austin2011introduction, stuart2010matching} to create balanced pseudo-populations. 
Recent DL methods adopt this principle and learn latent representations that predict outcomes while achieving a balance between treatment arms. Examples include RMSN~\cite{lim2018forecasting}, CRN~\cite{bica2020estimating}, TE-CED~\cite{seedat2022continuous}, and Causal-Transformer~\cite{melnychuk2022causal}.
These algorithms often perform poorly with finite samples in long horizons due to the compounding of small errors in the treatment models, leading to high variance \cite{bica2020estimating, petersen2012diagnosing, cole2008constructing}.

The \textit{outcome-model–based methods} paradigm directly models the outcome-generating process to estimate treatment effects using the \emph{G-computation} \cite{robins1986new,snowden2011implementation} by standardizing over the natural course of confounder distributions. Forward G-computation~\cite{taubman2009intervening, li2021g,chiu2023evaluating, mcgrath2020gformula,rein2024deep} simulates the time-varying confounder and outcomes based on the counterfactual treatments. While they often yield smaller variance~\cite{frauen2024model}, the need for a high-quality simulator is challenging in the finite-sample setting of healthcare. 
Recently, ICE G-computation gained attention due to the improved stability
from its recursive formulation. It decomposes the problem into a sequence of conditional expectation estimations (i.e., step-wise $Q$-function estimation), thereby avoiding extensive Monte Carlo rollouts and reducing the accumulation of simulation errors in predicting time-varying covariates across long horizons~\cite{frauen2023estimating, oprescu2025gst, liu2025bayesian,frauen2024model,hess2024igc}. 

The \emph{doubly robust methods}, combining treatment and outcome models, enhance robustness against model misspecification, improving finite-sample performance in high-dimensional, long-horizon settings. These include DML~\cite{lewis2021double,bodory2022evaluating}, LTMLE \cite{van2011targeted,stitelman2012general,rosenblum2010targeted,diaz2023nonparametric,shirakawa2024longitudinal, frauen2023estimating,guo2024estimating,frauen2024model} 
and SDR~\cite{luedtke2017sequential,diaz2023nonparametric}. We use recent SDR advances to augment ICE G-computation, then apply
LTMLE for final estimation.

Off-Policy Evaluation in RL
shares a similar statistical structure with longitudinal effect estimation.
Hence, OPE approaches mirror the standard categories in longitudinal causal estimation 
:
1) \emph{importance sampling} \cite{mahmood2014weighted, thomas2016data, hanna2019importance, schlegel2019importance} corresponds to IPTW and suffers from the high variance~\cite{liu2018breaking}.
2) \textit{Direct methods} (analogous to G-computation) simulate counterfactual trajectories
\cite{le2019batch,duan2020minimax,voloshin2021empiricalstudyoffpolicypolicy}. 3) \textit{DR methods} improve finite-sample efficiency and retain consistency when either component is correctly specified~\cite{bibaut2019more,pmlr-v48-jiang16,pmlr-v48-thomasa16,JMLR:v21:19-827,DBLP:journals/corr/abs-1103-4601}. 

\vspace{3pt}
\textbf{\emph{DL for Sequential Estimation.}}\;
Within the ICE G computation family, recent DL approaches exploit parameter sharing over time by using a single autoregressive model (e.g., LSTMs or transformers) 
to estimate $Q_t$~\cite{frauen2023estimating,oprescu2025gst,shirakawa2024longitudinal}. This design amortizes representation learning across steps while retaining 
the step-wise definition for $Q_t$.
However, operating within the recursive structure in ICE, this sharing induces a `moving-target' phenomenon that is well known in deep RL~\cite{mnih2013playing,fu2019diagnosing,lillicrap2015continuous}. 
A standard remedy is the \emph{target network}, a slowly updated copy of the major network, used to generate regression targets.
Additionally, as the horizon grows, early-step targets in ICE G-computation can become increasingly noisy due to compounding approximation errors, motivating auxiliary supervision to regularize representations and improve learning from weak or noisy signals~\cite{jaderberg2016reinforcement,fang2023predictive,stooke2021decoupling}.

Our work is situated at the intersection of these fields. We focus on the DL implementation of the theoretically sound ICE G-computation framework. 
Inspired by stabilization techniques from deep RL, we employ a target network to smooth the recursive learning process and use auxiliary supervision to regularize representation learning. Furthermore, we uniquely combine the SDR transformation (to reduce error propagation in nuisance models) with a final LTMLE step (to ensure statistical robustness), creating a novel two-stage debiasing procedure tailored for stable and efficient deep longitudinal estimation.

\section{Problem Setup and Preliminary}
We consider $n$ i.i.d. longitudinal samples, each observed over $\tau$ discrete time steps and generated under an unknown treatment regime. For each time point $t \in \{1,\ldots,\tau\}$, we observe time-varying covariates $L_t$, followed by a binary treatment assignment $A_t \in \mathcal{A}=\{0,1\}$. A terminal outcome $Y$ is observed after the final time point $\tau$. Accordingly, the observed data for a single unit can be written as
$O = (L_1, A_1, L_2, A_2, \ldots, L_\tau, A_\tau, Y) \sim P_0,$
where $P_0$ is the unknown true data-generating distribution. Any intermediate outcome observed after $A_t$ can be absorbed into $L_{t+1}$ without loss of generality.

For notational convenience, we use subscripts to represent temporal sequences, e.g., $X_{t:\tau} = (X_t, X_{t+1}, \ldots, X_\tau)$. The information available just before assigning treatment at time $t$ is summarized by the history
$
H_t = \{L_i, A_i\}_{i=1}^{t-1} \cup \{L_t\}\in \mathcal{H}_t.
$
Both covariates and treatments may depend on the entire preceding history: $L_t$ can be influenced by $H_{t-1} \cup \{A_{t-1}\}$, while treatment assignment at time $t$ follows an unknown mechanism characterized by
\begin{equation}
    G_t(H_t) = P(A_t = 1 \mid H_t),\label{eq:def_G}
\end{equation}
where $G_t : \mathcal{H}_t \to (0,1)$ denotes the propensity of receiving treatment conditional on the observed history.




\paragraph{\textbf{Problem Setup}}
We aim to estimate 
the potential outcome of a counterfactual treatment sequence, such that decisions can be informed by comparing the potential outcomes of different sequences.
Let $\ba = (a_1, \dots, a_{\tau})$ denote a generic counterfactual treatment sequence,
and $Y(\ba)$ be the corresponding potential outcome.
The \textbf{conditional average potential outcome (CAPO)} for sequence $\ba$ given baseline covariates $L_1$ is defined as
$\psi = \mathbb{E}[Y(\ba)|L_1].$


Throughout, we assume the standard longitudinal causal inference conditions to ensure the identifiability of our target parameters~\cite{frauen2023estimating}:
(1) Consistency: $Y(\ba) = Y$ when $A_{1:\tau}=\ba$;
(2) Sequential Ignorability: $Y(\ba) \perp A_t \mid H_t, \, \forall t \in \{1, \dots, \tau\}$;
(3) Positivity: $\forall t \in \{1, \dots, \tau\}, \, P(A_t=a_t|H_t)>0$.

The rest of the section reviews
three key methodological building blocks. We first describe ICE G‑computation (Sec. \ref{subsec:ice_g_comp}), which constructs plug‑in estimates of longitudinal treatment effects. We then briefly review semi-parametric efficiency theory (Sec. ~\ref{subsec:semi-parametric}). Grounded on the theory, we review LTMLE (Sec.~\ref{subsec:tmle}) and SDR estimator (Sec.~\ref{subsec:sdr}). They are both debiasing procedures that correct the first‑order plug‑in bias of initial estimators, including those from ICE, but differ in when the debiasing happens.

\subsection{ICE G-computation}\label{subsec:ice_g_comp} 
CAPO can be identified by the G-computation formula, which is a nested sequence of conditional expectations, evaluated under the counterfactual sequence $\ba$ \cite{robins2008estimation, bang2005doubly}:
\begin{align}
\psi = \mathbb{E} \Big[ \mathbb{E}\Big[ \dots
\mathbb{E} [Y | A_\tau = a_\tau, H_\tau] \dots ]| A_2 = a_2, H_2 \Big]| A_1 = a_1, L_1 \Big].
\end{align}


ICE G-computation estimates this quantity by breaking down the nested expectation into a series of recursive regression problems, proceeding backward from $t=\tau$ to $t=1$. Define $Q_{\tau+1} = Y$. For each step $t = \tau, \dots, 1$, ICE G-computation recursively defines 
\begin{equation}
    Q_t(A_t,H_t) = \mathbb{E}[Q_{t+1}(a_{t+1},H_{t+1})|A_t,H_t] \label{eq:temporal_equality}.
\end{equation}
The CAPO is then identified as
$\psi = \mathbb{E}[Q_1(a_1, H_{1})].$

Eq.~\eqref{eq:temporal_equality} implies that $Q_t$ can be learned by regressing on the conditional expectation of $Q_{t+1}(a_{t+1},H_{t+1})$, in the same spirit as temporal-difference learning in RL~\cite{sutton1998reinforcement}. In practice, the expectation is not computed explicitly. Instead, the realized value $Q_{t+1}(a_{t+1},H_{t+1})$ is directly used as the regression target for $Q_t$~\cite{shirakawa2024longitudinal,frauen2023estimating,oprescu2025gst}. This can be viewed as a one-sample Monte Carlo estimate of the conditional expectation. While these targets are noisy, minimizing the regression loss with such targets at the population level recovers the desired conditional expectation.

Putting these together, ICE G-computation is implemented by iterating the following two steps from $t = \tau$ to $1$:
\begin{enumerate}[itemsep=0pt, topsep=2pt, partopsep=0pt, parsep=0pt, leftmargin=*]
     \item Generate the regression target for step $t$ by computing the realized value at step $t+1$ under counterfactual treatment, $\hat{Q}_{t+1}(a_{t+1}, h_{t+1})$, a one-sample Monte Carlo simulation for the expectation in Eq.~\eqref{eq:temporal_equality}. (At $t=\tau$, this step is skipped because the target is $\hat{Q}_{\tau+1} = Y$ by construction.)
    \item Regress $Q_t(A_t,H_t)$ on the target $\hat{Q}_{t+1}(a_{t+1}, h_{t+1})$. 
\end{enumerate}
The estimate of $\psi$ is obtained by averaging $\hat{Q}_1(a_1, h_1)$ over samples. 


While ICE G‑computation provides a plug‑in estimator that is consistent under correct model specification, its theoretical guarantees become difficult to maintain when using highly flexible function approximators such as transformers, which may overfit or introduce regularization bias. This limitation motivates doubly robust alternatives that explicitly correct estimation bias.

\subsection{Semi-parametric Efficiency} \label{subsec:semi-parametric}

Unlike ICE G-computation, which relies solely on outcome regression, doubly robust methods explicitly leverage the treatment model to correct bias so that the estimator is consistent even when one of the nuisance models (outcome or treatment) is misspecified. Here, we review the semi-parametric efficiency theory before introducing the DR methods, including LTMLE and SDR.

Let the target parameter $\psi = \Psi(P)$ be a functional of the data-generating distribution $P$.
In practice, $\Psi(P)$ depends on nuisance components -- in our case, the propensity score and outcome model at each time step.
A plug-in estimator $\Psi(\hat P)$ is obtained by substituting estimates of these nuisances.
The error of this estimator can be expanded via a von Mises expansion \cite{van2011targeted, cho2024kernel}:
\begin{align}
    & \hat\psi_n - \psi_0 =\PP_n D_\Psi^*(P_0) - \underbrace{\PP_n D_\Psi^* (\hat P_n)}_{\text{plug-in bias}} + \notag \\ 
    &(\PP_n -P_0)[D_\Psi^*(\hat P_n)-D_\Psi^*(P_0)]+R_2(\hat P_n, P_0),\label{eq:von_mises}
\end{align}
where $D_{\Psi}^*(P)$ is the influence function (IF) of $\psi$ 
for the target parameter $\psi$ under the nonparametric model for the data distribution $P$. Meanwhile, $\mathbb P_n$, $P_0$ are the empirical and true expectation operators, respectively;  
At a high level, all DR methods, despite their different approaches, aim to achieve asymptotic efficiency by eliminating the \emph{plug-in bias term}, while relying on the remaining two terms to be $o_p(n^{-1/2})$ under regularity conditions~\cite{van2011targeted,chernozhukov2018double}.

\subsection{LTMLE}\label{subsec:tmle}
For our target estimand CAPO, its IF depends on the entire sequence of nuisance, $\{Q_t,G_t\}_{t=1}^\tau$ (defined in Eqs.~\eqref{eq:def_G} and \eqref{eq:temporal_equality}): 
\begin{equation}
    D_\Psi^*(P)(O) 
    = \sum_{s=1}^\tau\Bigg\{ \left(\prod_{k=1}^s w_k\right)
    \left[Q_{s+1}(a_{s+1},H_{s+1}) - Q_s(A_s,H_s)\right]\Bigg\},\label{eq:capo_eif}
\end{equation}
where $w_t$ is the inverse propensity score at time $k$
\begin{equation}
    w_k=\frac{\mathds{1}(A_k=a_{k})}{ a_{k}\cdot G_k(H_k) + (1-a_{k})(1-G_k(H_k))}.
\end{equation}
Guided by the IF, LTMLE eliminates the plug-in bias term by iteratively updating the initial estimates of $\{Q_t\}_{t=1}^\tau$, so as to solve the score equation $\mathbb{P}_n D_\Psi^*(\hat{P}_n) = 0$. We refer readers to prior work \citep{van2011targeted, cho2024kernel,li2025targeted, shirakawa2024longitudinal} for a detailed description of this procedure and Appendix~\ref{appx:ltmle} for our implementations.

Let $Q_t^0$ and $G^{0}_t$ denote the true outcome and propensity score model, and $\hat{Q}_t$, $\hat{G}_t$ their estimates. After the debiasing step, LTMLE provides the following multiply robust consistency:

\begin{lemma}[$\tau+1$ multiply robust consistency of LTMLE \cite{diaz2023nonparametric}]
    Assume that there is a time $k$ such that$ \|\hat{Q}_t - Q_t^0\| = o_p(1)$ for all $t>k$ and $\|\hat{G}_t - G^{0}_t\| = o_p(1)$ for all $t\leq k$. Then we have $\hat{\psi}^d = \psi^d + o_p(1)$.\label{lemma:ltmle_multiply_robust}
\end{lemma}

This result implies that consistency is achieved as long as all outcome regressions after some time point $k$ are consistent and all propensity score models at and before $k$ are consistent. For the $\tau$-step-ahead CAPO, this indicates $\tau+1$ chances to obtain a consistent estimate. This condition is substantially weaker than those required by pure IPTW or G-computation, which require consistency of all propensity scores or all outcome regressions, respectively.

This robustness structure reflects when LTMLE applies debiasing. 
In particular, LTMLE first estimates the entire sequence of outcome regressions $\{Q_t\}_{t=1}^\tau$ via ICE G-computation, and then applies the IF-guided correction at the end. Consequently, the success of the debiasing step depends on the quality of the entire initial estimate of outcome models. Because ICE G-computation proceeds recursively backward in time, a poor estimate at some time $k$ necessarily contaminates all downstream regressions $Q_t$ for $t<k$, limiting the effectiveness of post-hoc debiasing. 

\begin{algorithm}[t]
\caption{SDR~\cite{diaz2023nonparametric}}
\begin{algorithmic}[1]
    \STATE Input: Dataset $\{O_i\}_{i=1}^n$; Initial estimates of $G_t$ for $t = 1, \dots, \tau$
    \STATE Initialize $Q^\dag_{\tau + 1} = Y$
    \FOR{$t =\tau $ to $1$}
        \STATE (Learn) Use observed data to regress $Q_t(A_t, H_t)$ on $Q^\dag_{t+1}(a_{t+1}, H_{t+1})$
        \STATE (Predict) Generate the regression target for $Q_{t-1}(A_{t-1}, H_{t-1})$:
                \begin{align}
                    &Q^\dag_t(\dot{a}_t, h_t) = Q_t(a_t, h_t) +\notag\\
                    &\quad\quad\sum_{s=t}^\tau\left\{ \left(\prod_{k=t}^s w_k\right)
                    \left[Q_{s+1}(a_{s+1},H_{s+1}) - Q_s(A_s,H_s)\right]\right\}.\label{eq:sdr_target}
                \end{align}
    \ENDFOR
    \STATE Final estimate $\hat{\psi}_n = \mathbb{P}_nQ^\dag_1(a_1, H_1)$
\end{algorithmic}\label{algo:sdr}
\end{algorithm}

\subsection{SDR}\label{subsec:sdr}
The post-hoc debiasing strategy of LTMLE, while effective, is still not optimal. 
Recent developments in semi-parametric theory show that stronger robustness guarantees can be achieved.
In particular, SDR estimators (see Algorithm~\ref{algo:sdr}) achieve consistency under substantially weaker conditions by requiring that, at each time step, \emph{either} the outcome regression or the propensity score model be consistently estimated, leading to $2^\tau$ chances to achieve consistency:

\begin{lemma}[$2^\tau$ multiply robust consistency~\cite{diaz2023nonparametric}]
    Assume that, for each time $t$, either $ \|\hat{Q}_t - Q_t^\dagger\| = o_p(1)$ or $\|\hat{G}_t - G_t^0\| = o_p(1)$. Then we have $\hat{\psi}^d = \psi^d + o_p(1)$.
\end{lemma}
The fundamental difference between SDR and ICE G-computation lies in the regression target. Rather than regressing $Q_t$ on the plug-in pseudo-outcome $\hat{Q}_{t+1}(a_{t+1}, H_{t+1})$, SDR regresses on a debiased target, $Q^\dagger_{t+1}(a_{t+1}, H_{t+1})$, that augments the plug-in estimate with an IF-based correction, as shown in Eq.~\eqref{eq:sdr_target} in Alg.~\ref{algo:sdr}. This leads to a crucial difference in the required consistency conditions: LTMLE relies on consistent estimation of the true model $Q^0_t$ (Lemma~\ref{lemma:ltmle_multiply_robust}), whereas SDR only requires consistency between $\hat{Q_t}$ and its constructed regression target $Q^\dagger_{t+1}$. As a result, SDR shifts the burden from accurately recovering the true outcome model to accurately fitting a regression, yielding a weaker requirement in finite samples.

To see the intention behind SDR, note that the augmentation term in
Eq.~\eqref{eq:sdr_target} shares the same functional form as the IF in Eq.~\eqref{eq:capo_eif}. This is not a coincidence: both are derived from solving the score equation for CAPO, differing only in the time point at which this statistical estimation problem is anchored. ICE G-computation identifies the final $\tau$-step-ahead CAPO, $\mathbb{E}[Q_1(a_1, H_1)]$, by recursively modeling outcomes from the final time point backward. In contrast, the SDR transformation treats each intermediate quantity $\mathbb{E}[Q_t(a_t, H_t)]$ as a $(\tau - t)$-step-ahead CAPO in its own right. Consequently, the augmentation term at step $t$ is precisely the IF for that shorter-horizon estimand, thereby embedding local double robustness into the sequential regression updates.
\begin{figure}[t]
    \centering
    \includegraphics[width=\linewidth]{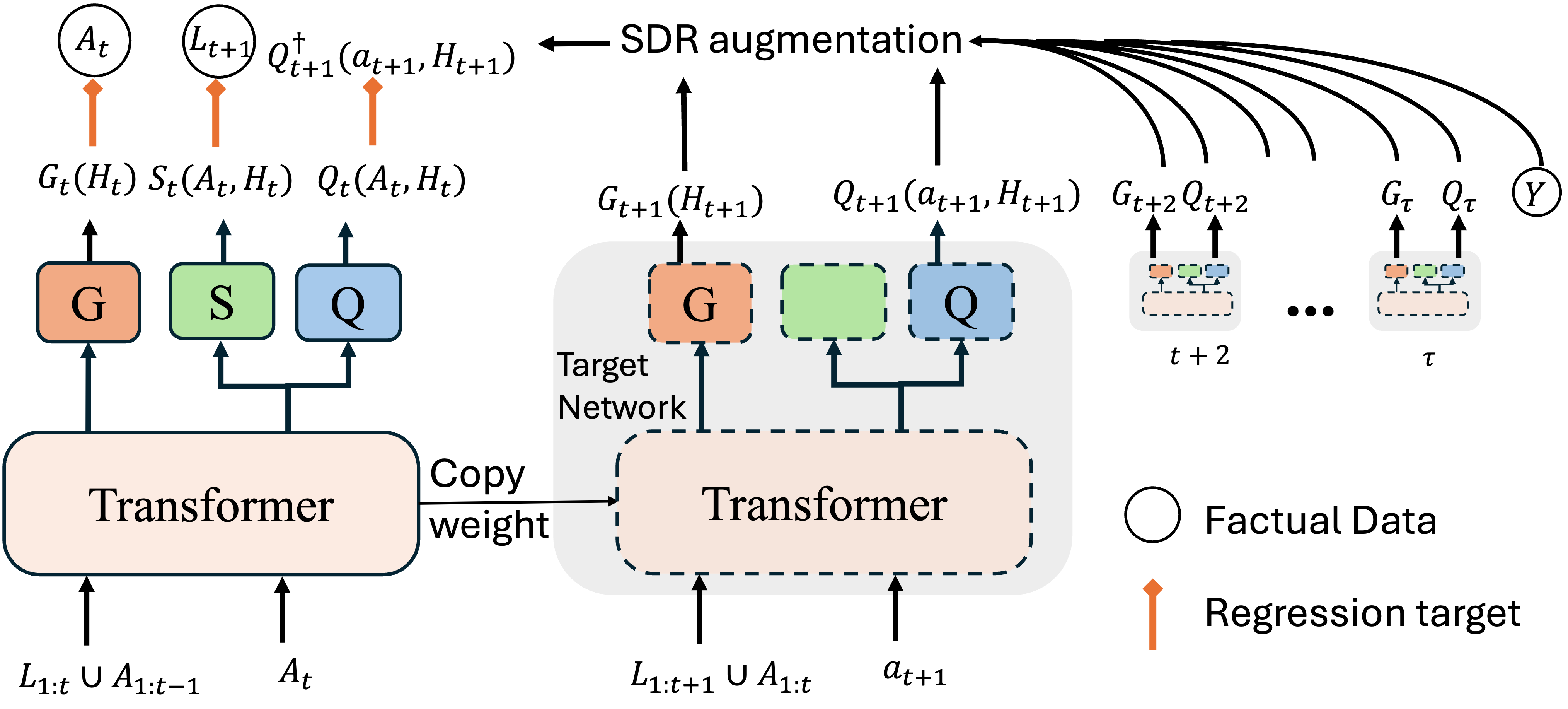}
    \caption{Overview of the $D^3$-Net architecture and training procedure.
$D^3$-Net uses a shared multi-task transformer backbone with the outcome ($Q$), treatment ($G$), and simulator ($S$) heads. At each time step $t$, the $Q$-head is trained using SDR-augmented targets, $Q_{t+1}^\dagger(a_{t+1}, H_{t+1})$, which combines future outcome and treatment models. A target network, implemented as a delayed copy of the main network, is used to generate stable SDR targets for recursive learning.}
    \label{fig:d3-net}
    \vspace{-\baselineskip}
\end{figure}


Formally, at time $t$,
the influence function for the $(\tau - t)$-step-ahead CAPO is given by
\begin{equation*}
    D_\Psi^*(P)(O) 
    = \sum_{s=t}^\tau\left\{ \left(\prod_{k=t}^s w_k\right)
    \left[Q_{s+1}(a_{s+1},H_{s+1}) - Q_s(A_s,H_s)\right]\right\},
\end{equation*}
matching the augmentation term in
Eq.~\eqref{eq:sdr_target}. 
This construction is directly analogous to AIPTW for static CATE estimation, where debiasing is achieved by 
an additive IF correction~\cite{scharfstein1999adjusting,robins1994estimation}.

Although both LTMLE and SDR are multiply robust and rooted in the same semi-parametric efficiency theory, the key distinction between them lies in \textbf{when and how debiasing is applied}. LTMLE first fits the entire sequence of outcome regressions using ICE G-computation and then performs a \emph{single, final debiasing step}, enforcing the score equation only for the $\tau$-step-ahead CAPO. As a result, estimation errors may accumulate during the backward recursion before they are corrected.

In contrast, SDR \emph{debiases recursively at every time step}. It treats each intermediate regression $Q_t$ as a shorter-horizon CAPO estimand and enforces its corresponding score equation immediately. This stepwise fitting and debiasing prevents upstream errors from propagating downstream, yielding \emph{sequential double robustness} throughout the estimation process.
However, because its debiasing procedure directly incorporates a product of inverse propensity scores in the objective function at each step, it is susceptible to extreme weights and high variance in finite samples, particularly when propensity estimates are near the boundaries.

LTMLE does not exhibit this unstable behavior. Its debiasing is performed through a fluctuation of the outcome model within a parametric submodel, which ensures that the updated estimate naturally respects the support of the outcome and avoids the explosive variance associated with propensity score products.
This tension, between SDR’s recursive protection against error propagation and its potential for variance inflation, motivates our hybrid approach.

\section{Methods}

Our method consists of two stages: training and inference. 
During training, we use an SDR-augmented objective to embed debiasing into outcome regression and stabilize optimization with a target network (Section~\ref{subsec:sdr_gcomp}); we additionally regularize representation learning through auxiliary supervisions based on time-varying covariate prediction, and present the overall model architecture in Section~\ref{subsec:architecture}. 
In the inference stage, we discard the SDR debiasing term and instead apply a final LTMLE targeting step, combining the strengths of both SDR and LTMLE; the resulting inference procedure and overall algorithm are described in Section~\ref{subsec:reverse_sdr}.

\begin{algorithm}[t]
    \caption{$D^3$-Net}\label{algo:deepsdr}
    \begin{algorithmic}[1]
        \STATE \underline{\textbf{Training:}}
        \STATE \textbf{Input:} Data $\{O_i\}_{i=1}^n$; horizon $\tau$; target network update rate $\beta$.
        \STATE Initialize $\theta$; set target copy $\theta' \gets \theta$.
        \FOR{batch $=1,\dots,B$}
            \STATE \textbf{(Major model forward)} For $t=1,\dots,\tau$:
            \STATE $Q_t(A_t,H_t) \gets h_Q(\phi(H_t,A_t;\theta_\phi);\theta_Q)$
            \STATE $G_{t}(H_t) \gets h_{G}(\phi(H_t;\theta_\phi);\theta_{G})$
            \STATE $S_{t}(A_t,H_t) \gets h_{S}(\phi(H_t,A_t;\theta_\phi);\theta_{S})$
        
            \STATE \textbf{(Target network forward for SDR)} For $t=1,\dots,\tau$:
            \STATE $Q'_t(A_t,H_t) \gets h_Q(\phi(H_t,A_t;\theta'_\phi);\theta'_Q)$
            \STATE $Q'_t(a_t,H_t) \gets h_Q(\phi(H_t,a_t;\theta'_\phi);\theta'_Q)$
            \STATE $G'_{t}(H_t) \gets h_{G}(\phi(H_t;\theta'_\phi);\theta'_{G})$
        
            \STATE $Q^\dagger_{\tau+1} \gets Y$
            \FOR{$t=\tau,\dots,1$}
                \STATE Compute $Q^\dagger_{t+1}(a_{t+1},H_{t+1})$ via SDR transformation (Eq.~\eqref{eq:sdr_target}) using $\{Q'_s(\cdot),G'_{s}(\cdot)\}_{s=t+1}^{\tau}$.
            \ENDFOR
        
            \STATE $\mathcal{L}_Q \gets \sum_{t=1}^{\tau}\big(Q_t(A_t,H_t)-Q^\dagger_{t+1}(a_{t+1},H_{t+1})\big)^2$
            \STATE $\mathcal{L}_G \gets \sum_{t=1}^{\tau}\,L_{\mathrm{bce}}\!\left(A_{t},\,G_{t}(H_t)\right)$
            \STATE $\mathcal{L}_S \gets \sum_{t=1}^{\tau}\,(S_{t}(A_t,H_t)-L_{t+1})^2$
            \STATE $\mathcal{L}\gets \mathcal{L}_Q + \alpha(\mathcal{L}_G+\mathcal{L}_S)$
        
            \STATE $\theta \gets \theta - \eta\nabla_\theta \mathcal{L}$
            \STATE $\theta' \gets \beta\theta + (1-\beta)\theta'$
        \ENDFOR
        \STATE \textbf{Output:} parameters $\theta=(\theta_\phi,\theta_Q,\theta_G,\theta_S)$.
        \STATE
        \STATE \underline{\textbf{Inference:}}
        \STATE \textbf{Input:} Evaluation Data $\{O_i\}_{i=1}^m$.
        \FOR{$t = 1, \dots, \tau$}
            \STATE $G_{t}(H_t) = h_{G}(\phi(H_t;\theta_\phi);\theta)$
            \STATE $Q_t(A_t,H_t) = h_Q(\phi(H_t,A_t;\theta_\phi);\theta_Q)$
        \ENDFOR
        \STATE \textbf{Output:} $\hat{\psi} = LTMLE(\{O_i\}_{i=1}^m;\{Q_t\}_{t=1}^\tau; \{G_{t}\}_{t=1}^\tau)$
    \end{algorithmic}
    \vspace{-2pt}
\end{algorithm}

\subsection{Debias During Training via SDR}\label{subsec:sdr_gcomp}
We first describe the training objective that integrates SDR-style debiasing into outcome regression.
At a high level, we parameterize a sequence of outcome and treatment models $\{Q_t,G_t\}_{t=1}^\tau$, where $Q_t$ are trained with supervision defined in SDR and $G_t$ are supervised by factual treatment assignments.
Concretely, for each time step $t$, we construct a debiased pseudo-outcome $\hat{Q}^\dagger_{t+1}(a_{t+1}, H_{t+1})$ following Eq.~\eqref{eq:sdr_target}. 
The $Q$-function is then trained by minimizing
$\mathcal{L}_Q = \sum_{t=1}^\tau \mathbb{E}\left[(Q_t(A_t,H_t) - \hat{Q}^\dag_{t+1}(a_{t+1},H_{t+1}))^2\right].$

Analogous to the discussion of ICE (Section~\ref{subsec:ice_g_comp}), this objective treats the realized value of $\hat{Q}^\dag_{t+1}(a_{t+1},H_{t+1})$ as a one sample Monte Carlo simulation for $\mathbb{E}[\hat{Q}^\dag_{t+1}(a_{t+1},H_{t+1})|A_t,H_t]$. Minimizing this objective at the population level recovers the SDR-defined supervision: $\hat{Q}_{t}(A_{t},H_{t})=\mathbb{E}[\hat{Q}^\dag_{t+1}(a_{t+1},H_{t+1})|A_t,H_t]$. Meanwhile, the $G$-function is trained by minimizing
$\sum_{t=1}^\tau L_\mathrm{bce}(A_t,G_t(H_t)),$
where $L_\mathrm{bce}(x,\hat x)= x\log (\hat x) + (1-x)\log(1 -\hat x)$.

 

We parametrize $\{Q_t\}_{t=1}^\tau$ using a single, shared neural network. 
This allows the model to exploit the temporal structure in data, improving sample efficiency. However, in the recursive training of ICE or SDR, this shared parameterization introduces a \textbf{moving-target}: at each iteration, $Q_t$ is trained to regress on a pseudo-outcome $\hat{Q}_{t+1}^\dag$ that is itself produced by the same, changing parameters. After each gradient step, these shared parameters are updated, causing the regression target to shift, which leads to unstable training dynamics. The SDR augmentation compounds this issue, as its debiasing term involves a product of propensity scores that can introduce additional variance and amplify target fluctuations.

To stabilize training, we employ two techniques. First, we incorporate a \emph{target network}, a standard heuristic from deep reinforcement learning. We maintain a slowly updated copy of the parameters via Polyak averaging (line 22, Algorithm~\ref{algo:deepsdr}) and use this lagged copy to compute the pseudo-outcome $Q^\dag_{t+1}$ for the regression loss. This decouples the target from immediate parameter updates, smoothing its evolution across training iterations.
Second, we truncate the inverse propensity product in the SDR debiasing term to 20 and clip the SDR-augmented regression targets to [0,1], since outcomes are min-max scaled to [0,1] in data preprocessing. 

\subsection{Model Architecture}
\label{subsec:architecture}


Figure~\ref{fig:d3-net} illustrates the overall architecture of D$^3$-Net. The model is built upon the heterogeneous transformer architecture introduced in DeepLTMLE~\cite{shirakawa2024longitudinal}. Inputs are embedded according to variable type and processed by a transformer encoder to obtain a shared longitudinal representation.
This shared representation is used by three prediction heads. The outcome head ($Q$-head) and the treatment head ($G$-head) estimate the outcome regressions and treatment assignment required by the SDR objective in Section~\ref{subsec:sdr_gcomp}. In addition, we include an auxiliary covariate prediction head ($S$-head) that predicts the next-step covariates $L_{t+1}$ given the current history and treatment.
The $S$-head is trained using
$
\mathcal{L}_S =
\sum_{t=1}^{\tau}
\left(
S_t(A_t,H_t)-L_{t+1}
\right)^2,
$
which provides additional supervision on patient-state dynamics and regularizes the shared representation.

All heads are implemented as two-layer MLPs. The $Q$- and $G$-heads use sigmoid outputs to enforce predictions in $[0,1]$,\footnote{The $Q$-head is trained with an MSE loss between the sigmoid-bounded prediction and the pseudo-outcome.} while the $S$-head uses a linear output layer for continuous covariates.
At step $t$, $G_t$ receives history $H_t$, while $Q_t$ and $S_t$ receive $(H_t,A_t)$. This temporal ordering is enforced through causal masking, preventing access to future observations.

\begin{figure*}[!t]
    \centering
    \begin{subfigure}{0.49\linewidth}
        \includegraphics[width=1\linewidth]{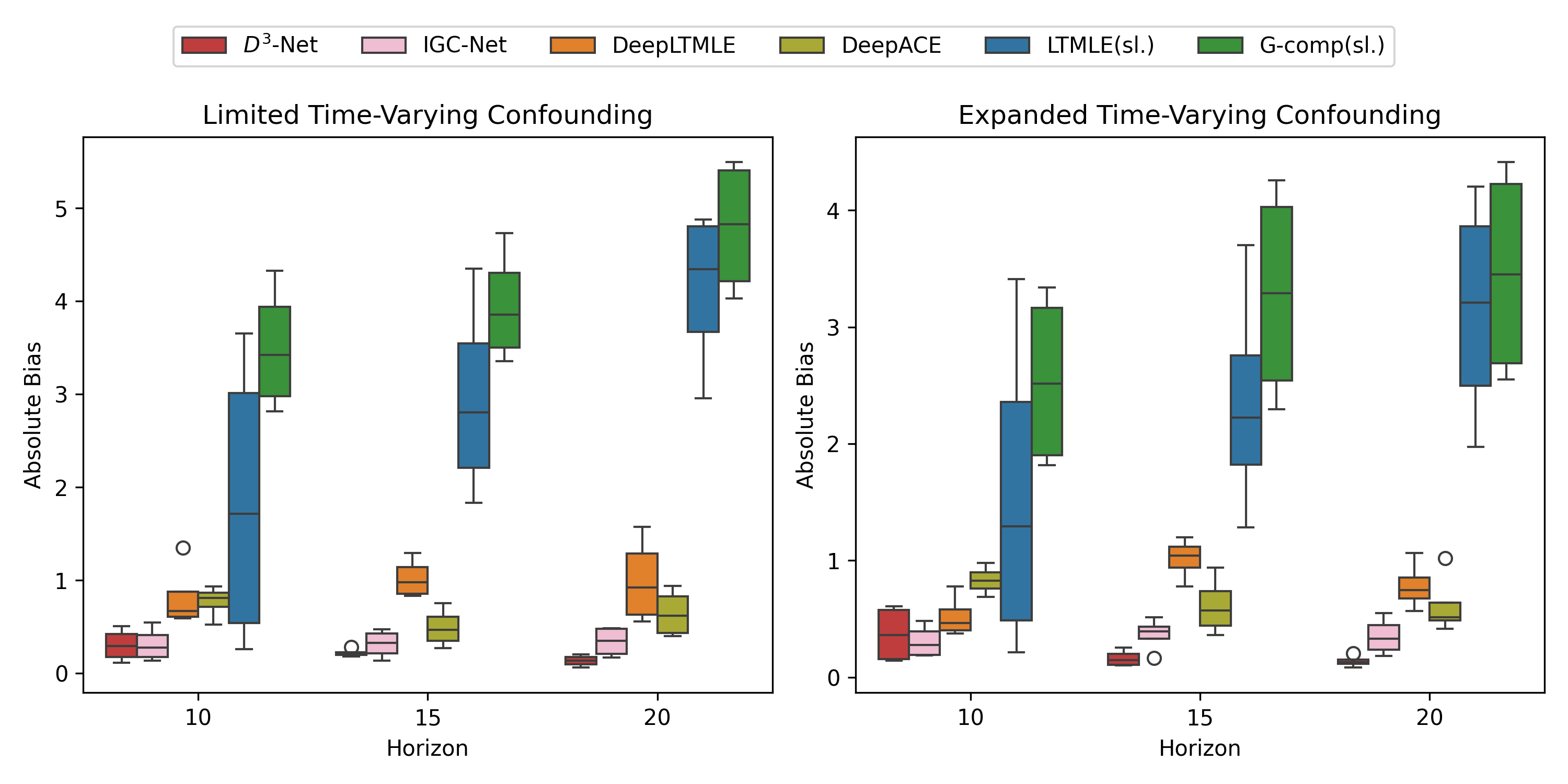}
    \end{subfigure}
    \hfill
    \begin{subfigure}{0.49\linewidth}
        \centering
        \includegraphics[width=1\linewidth]{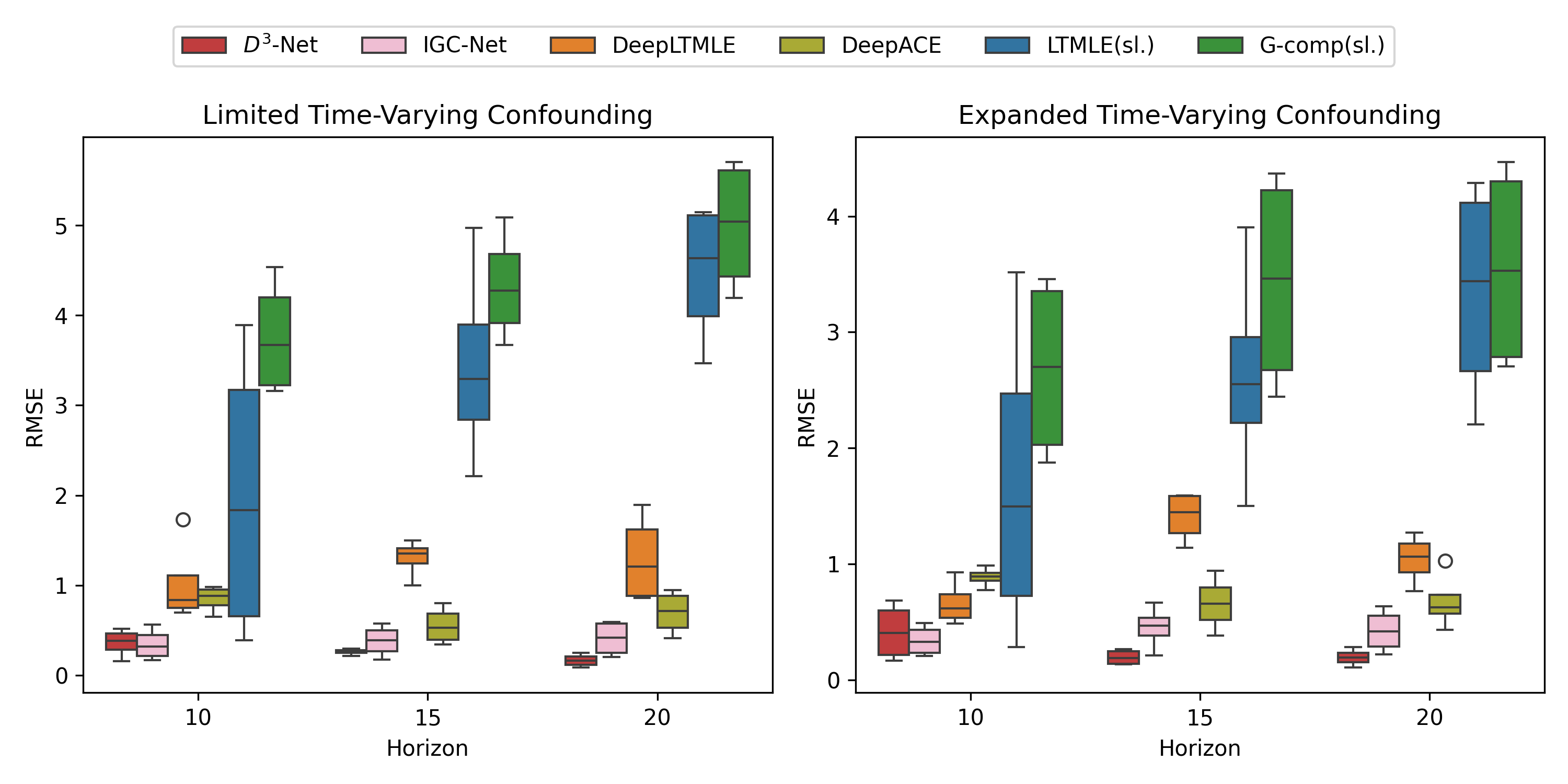}
    \end{subfigure}
    \vspace{-4mm}   
    \caption{Overall performance at horizons $\tau \in \{10,15,20\}$ under different levels of time-varying confounding. Left two: absolute error; right two: RMSE.}
    \label{fig:main_result}
    \vspace{-0.5\baselineskip}
\end{figure*}

\subsection{Re-debias via LTMLE}\label{subsec:reverse_sdr}
As discussed in Section~\ref{subsec:sdr}, SDR relies on correction terms that involve products of inverse propensity scores. While truncation can improve numerical stability, the resulting estimator may remain sensitive to errors in the nuisance estimates. To mitigate this issue, we adopt a hybrid strategy: after training, we drop the SDR correction terms and treat the uncorrected outcome models, $\hat{Q}_t(a_t, H_t)$, as estimated nuisance and perform LTMLE (line 27-32 in Algorithm~\ref{algo:deepsdr}; we describe LTMLE in Algorithm~\ref{appx:ltmle}). Intuitively, SDR training improves the quality of the initial outcome regressions by reducing the accumulation of upstream estimation errors, while LTMLE performs a final targeting step on these improved nuisance estimates. This combination allows us to retain the benefits of SDR training while reducing reliance on the direct SDR correction terms at inference time. The following lemma characterizes how the training targets induce either first-order (ICE) or second-order (SDR) dependence on upstream nuisance errors, clarifying why SDR initialization can yield better initial outcome regressions for subsequent LTMLE targeting (Proof in Appendix~\ref{appx:first_vs_second_order}). 

\begin{restatable}[First- vs.\ second-order dependence in training targets]{lemma}{LemmaFirstSecondOrder}\label{lemma:first_vs_second_order}
Assume bounded outcomes and positivity. Let $\hat Q_t^{\mathrm{ICE}}$ be 
the initial estimate from
regular ICE G-computation,
and let $\hat Q_t^{\mathrm{SDR}}$ be 
the initial estimate from SDR procedure.
Throughout, we suppress function arguments when clear from context; e.g., $Q_t^0$ denotes $Q_t^0(A_t,H_t)$ and $G_k^0$ denotes $G_k^0(A_k\mid H_k)$ unless arguments are shown explicitly.

Then, for each $t=1,\dots,\tau-1$, the population regression errors admit the bounds
\begin{align*}
    \|\hat Q_t^{\mathrm{ICE}} - Q_t^0\|_2 \;&\le\; \xi_t^{\mathrm{ICE}} + \|\hat Q^{\mathrm{ICE}}_{t+1}(a_{t+1},H_{t+1}) - Q_{t+1}^0(a_{t+1},H_{t+1})\|_2,\\
\|\hat Q_t^{\mathrm{SDR}} - Q_t^0\|_2 \;&\le\; \xi_t^{\mathrm{SDR}} + \sum_{k=t+1}^{\tau} C_{t,k}\,
\|\hat G^{\mathrm{SDR}}_k - G_k^0\|_2\,\|\hat Q^{\mathrm{SDR}}_k - Q_k^\dagger\|_2,
\end{align*}
for constants $C_{t,k}$ depending only on positivity and $\tau$, where $\xi_t^{\mathrm{ICE}}$ and $\xi_t^{\mathrm{SDR}}$ denote the regression-to-target error at time $t$ for corresponding estimator.
\end{restatable}

This lemma shows that, under regular ICE G-computation, the error in $\hat Q^{\mathrm{ICE}}_t$ exhibits first-order dependence on upstream outcome regression error. We note that $\|\hat Q^{\mathrm{ICE}}_{t+1}(a_{t+1},H_{t+1}) - Q_{t+1}^0(a_{t+1},H_{t+1})\|_2$ is evaluated at $(a_{t+1},H_{t+1})$ and thus this bound does not directly telescope into a horizon-wide inequality under the same shorthand norm. In contrast, using training targets debiased by SDR transformation yields a second-order dependence on upstream outcome and propensity-score estimation errors. Notably, the product doesn't involve the next-step error to the true model ($\|\hat Q^{\mathrm{SDR}}_k - Q_k^0\|_2$), but involves the error to the constructed target ($\|\hat Q^{\mathrm{SDR}}_k - Q_k^\dagger\|_2$). Consequently, the effect of upstream nuisance estimation enters only at second order, which makes the error propagation across time more amenable to control. $D^3$-Net couples this improved nuisance estimation with LTMLE targeting, making the multiply robust consistency condition of LTMLE (Lemma~\ref{lemma:ltmle_multiply_robust}) easier to satisfy in practice.

\section{Experiments}

\begin{table*}[!t]
\centering
\caption{Absolute bias (mean $\pm$ std) and RMSE of CAPO estimates at horizon $\tau = 15$.}
\vspace{-10pt}
\label{tab:results-tau15}
\resizebox{0.95\linewidth}{!}{%
\begin{tabular}{l
 cccc cccc
 cccc cccc}
\toprule
& \multicolumn{8}{c}{\textbf{Limited Time-varying Confounding}} 
& \multicolumn{8}{c}{\textbf{Expanded Time-varying Confounding}} \\
\cmidrule(lr){2-9} \cmidrule(lr){10-17}

\textbf{Model}
& \multicolumn{4}{c}{\textbf{Bias}}
& \multicolumn{4}{c}{\textbf{RMSE}}
& \multicolumn{4}{c}{\textbf{Bias}}
& \multicolumn{4}{c}{\textbf{RMSE}} \\

& CF 1 & CF 2 & CF 3 & CF 4
& CF 1 & CF 2 & CF 3 & CF 4
& CF 1 & CF 2 & CF 3 & CF 4
& CF 1 & CF 2 & CF 3 & CF 4 \\
\midrule

\rowcolor{gray!6}
G-comp (sl.)
& 3.55$\pm$1.88 & 4.73$\pm$1.91 & 3.35$\pm$1.54 & 4.16$\pm$1.90
& 3.99 & 5.09 & 3.67 & 4.55
& 2.62$\pm$0.85 & 3.95$\pm$1.38 & 2.30$\pm$0.85 & 4.26$\pm$1.01
& 2.75 & 4.17 & 2.44 & 4.37 \\

LTM. (sl.)
& 3.28$\pm$1.37 & 4.35$\pm$2.47 & 1.83$\pm$1.27 & 2.33$\pm$2.02
& 3.54 & 4.97 & 2.21 & 3.05
& 2.44$\pm$1.03 & 3.70$\pm$1.27 & 1.28$\pm$0.79 & 2.00$\pm$1.46
& 2.64 & 3.90 & 1.50 & 2.46 \\

\rowcolor{gray!6}
D.ACE
& 0.75$\pm$0.27 & 0.37$\pm$0.17 & 0.56$\pm$0.33 & 0.27$\pm$0.21
& 0.80 & 0.41 & 0.65 & 0.34
& 0.94$\pm$0.11 & 0.36$\pm$0.12 & 0.67$\pm$0.34 & 0.47$\pm$0.32
& 0.94 & 0.38 & 0.75 & 0.56 \\

D.LTMLE
& 0.87$\pm$0.51 & 1.09$\pm$0.88 & 1.30$\pm$0.77 & 0.83$\pm$1.06
& 1.00 & 1.39 & 1.50 & 1.32
& 0.78$\pm$0.86 & 0.99$\pm$0.87 & 1.20$\pm$1.06 & 1.09$\pm$1.19
& 1.14 & 1.31 & 1.58 & 1.59 \\

\rowcolor{gray!6}
IGC-Net
& \underline{0.41$\pm$0.25} & \textbf{0.14$\pm$0.11} & \underline{0.47$\pm$0.33} & \underline{0.24$\pm$0.17}
& \underline{0.48} & \textbf{0.17} & \underline{0.57} & \underline{0.30}
& \underline{0.51$\pm$0.44} & \underline{0.17$\pm$0.13} & \underline{0.40$\pm$0.29} & \underline{0.38$\pm$0.23}
& \underline{0.67} & \underline{0.21} & \underline{0.49} & \underline{0.44} \\

$D^3$-Net
& \textbf{0.18$\pm$0.12} & \underline{0.20$\pm$0.17} & \textbf{0.28$\pm$0.09} & \textbf{0.20$\pm$0.19}
& \textbf{0.21} & \underline{0.26} & \textbf{0.30} & \textbf{0.27}
& \textbf{0.11$\pm$0.09} & \textbf{0.10$\pm$0.09} & \textbf{0.25$\pm$0.09} & \textbf{0.18$\pm$0.16}
& \textbf{0.14} & \textbf{0.14} & \textbf{0.27} & \textbf{0.24} \\

\bottomrule
\end{tabular}
}
\vspace{-0.5\baselineskip}
\end{table*}

\begin{table*}[!t]
\centering
\caption{Absolute bias (mean $\pm$ std) and RMSE of CAPO estimates at horizon $\tau = 20$.}
\vspace{-10pt}
\label{tab:results-tau20}
\resizebox{0.95\linewidth}{!}{%
\begin{tabular}{l
 cccc cccc
 cccc cccc}
\toprule
& \multicolumn{8}{c}{\textbf{Limited Time-varying Confounding}} 
& \multicolumn{8}{c}{\textbf{Expanded Time-varying Confounding}} \\
\cmidrule(lr){2-9} \cmidrule(lr){10-17}

\textbf{Model}
& \multicolumn{4}{c}{\textbf{Bias}}
& \multicolumn{4}{c}{\textbf{RMSE}}
& \multicolumn{4}{c}{\textbf{Bias}}
& \multicolumn{4}{c}{\textbf{RMSE}} \\

& CF 1 & CF 2 & CF 3 & CF 4
& CF 1 & CF 2 & CF 3 & CF 4
& CF 1 & CF 2 & CF 3 & CF 4
& CF 1 & CF 2 & CF 3 & CF 4 \\
\midrule

\rowcolor{gray!6}
G-comp (sl.)
& 4.27$\pm$1.48 & 5.38$\pm$1.54 & 4.03$\pm$1.20 & 5.49$\pm$1.57
& 4.51 & 5.58 & 4.19 & 5.70
& 2.55$\pm$0.92 & 4.16$\pm$0.85 & 2.74$\pm$0.66 & 4.41$\pm$0.71
& 2.70 & 4.24 & 2.81 & 4.47 \\

LTM. (sl.)
& 3.91$\pm$1.49 & 4.88$\pm$1.69 & 2.95$\pm$1.86 & 4.78$\pm$1.84
& 4.17 & 5.15 & 4.19 & 5.70
& 2.67$\pm$0.91 & 4.20$\pm$0.87 & 1.97$\pm$1.01 & 3.75$\pm$1.60
& 2.82 & 4.29 & 2.20 & 4.06 \\

\rowcolor{gray!6}
D.ACE
& 0.94$\pm$0.15 & 0.40$\pm$0.10 & 0.79$\pm$0.35 & \underline{0.44$\pm$0.37}
& 0.95 & 0.41 & 0.86 & \underline{0.57}
& 1.02$\pm$0.12 & 0.42$\pm$0.12 & 0.51$\pm$0.35 & 0.51$\pm$0.38
& 1.03 & 0.43 & 0.62 & 0.63 \\

D.LTMLE
& 0.65$\pm$0.57 & 1.58$\pm$1.08 & 0.56$\pm$0.71 & 1.19$\pm$0.98
& 0.86 & 1.90 & 0.89 & 1.53
& 0.57$\pm$0.53 & 1.07$\pm$0.71 & 0.71$\pm$0.70 & 0.78$\pm$0.86
& 0.77 & 1.27 & 0.98 & 1.15 \\

\rowcolor{gray!6}
IGC-Net
& \underline{0.48$\pm$0.30} & \underline{0.17$\pm$0.11} & \underline{0.22$\pm$0.15} & 0.48$\pm$0.36
& \underline{0.57} & \textbf{0.20} & \underline{0.27} & 0.59
& \underline{0.55$\pm$0.33} & \underline{0.18$\pm$0.13} & \underline{0.25$\pm$0.18} & \underline{0.41$\pm$0.34}
& \underline{0.63} & \textbf{0.22} & \underline{0.31} & \underline{0.53} \\

$D^3$-Net
& \textbf{0.11$\pm$0.07} & \textbf{0.16$\pm$0.12} & \textbf{0.06$\pm$0.06} & \textbf{0.20$\pm$0.15}
& \textbf{0.13} & \textbf{0.20} & \textbf{0.09} & \textbf{0.25}
& \textbf{0.13$\pm$0.11} & \textbf{0.13$\pm$0.18} & \textbf{0.08$\pm$0.07} & \textbf{0.21$\pm$0.20}
& \textbf{0.17} & \textbf{0.22} & \textbf{0.11} & \textbf{0.28} \\

\bottomrule
\end{tabular}
}
\vspace{-0.5\baselineskip}
\end{table*}

We assess $D^3$-Net using semi-synthetic benchmarks constructed from MIMIC-III~\cite{PhysioNet-mimiciii-1.4}, a real-world critical care dataset. 
Section~\ref{subsec:setup} describes the experimental design. Building on the benchmarking configuration of DeepACE, we further evaluate our methods under stronger time-varying confoundings and various horizons. Section~\ref{subsec:results} reports quantitative results, where $D^3$-Net demonstrates consistent lower bias and variance in these challenging settings. 
Section~\ref{subsec:ablation} presents ablation studies.

\subsection{Setup}\label{subsec:setup}

\paragraph{{Baselines}}
We benchmark $D^3$-Net against ICE-based longitudinal estimators, including (1) G-computation with Super Learner, (2) LTMLE with Super Learner~\cite{lendle2017ltmle}, (3) DeepACE \cite{frauen2023estimating}, (4) DeepLTMLE \cite{shirakawa2024longitudinal}, and (5) IGC-Net~\cite{hess2024igc}. Methods (3)-(5) are ICE G-computation-based but differ in their architecture and debiasing strategies: DeepACE is LSTM-based and incorporates an influence-function-based objective; DeepLTMLE is transformer-based and applies LTMLE as a post-training correction; IGC-Net
employs a more sophisticated multi-transformer architecture without explicit debiasing.
We include IGC-Net primarily as an extra benchmark. As we will show, IGC‑Net’s strong architecture often outperforms existing debiasing methods (DeepACE, DeepLTMLE), yet $D^3$-Net consistently outperforms IGC‑Net, demonstrating the value of our SDR‑based training and final LTMLE correction. A direct, architecture‑controlled comparison (applying our debiasing to IGC‑Net’s backbone) would further isolate the benefit of debiasing; we leave this as a natural extension.
Implementation details are provided in Appendix~\ref{appx:implementation}.

\paragraph{{Settings}}
We begin with the semi-synthetic DGP proposed in DeepACE, which combines 10 real time-varying clinical measurements from MIMIC-III with synthetically generated treatment assignments and outcomes. In this setting, treatment affects the single outcome variable but does not affect any other covariates, yielding a limited treatment–confounder feedback mechanism. 
To study more challenging scenarios, we further augment this DGP with five additional synthetic time-varying covariates that both affect and are affected by treatment, thereby inducing an expanded and more realistic feedback structure. We refer to these two settings as \emph{Limited Time-varying Confounding} and \emph{Expanded Time-varying Confounding} throughout the paper. We also experimented with high-dimensional synthetic covariates in the expanded DGP. Full specifications of all DGPs are provided in Appendix~\ref{appx:dgp}.

For each DGP, we sample $N=1000$ patient trajectories and generate longitudinal data over different horizon lengths $\tau$. We begin with $\tau=15$, as chosen by DeepACE. Hyperparameter tuning procedures for all baselines are detailed in Appendix~\ref{appx:implementation}.


Because longitudinal estimators can exhibit different bias behavior under different counterfactuals, we evaluate each method across multiple counterfactual sequences rather than a single one. Similar to DeepACE, we evaluate four counterfactual (CF) sequences for each horizon: 
(1) CF 1 = $(0,\ldots,0)$; (2) CF 2 = $(1,\ldots,1)$;
(3) CF 3 =$(\underbrace{1,\ldots,1}_{s_\tau},
\underbrace{0,\ldots,0}_{\tau-s_\tau})$, 
(4) CF 4 = $(\underbrace{0,\ldots,0}_{\tau-s_\tau},
\underbrace{1,\ldots,1}_{s_\tau}),$
where $s_\tau=10$ for $\tau\in\{15,20\}$, $s_\tau=5$ for $\tau=10$.
All experiments are repeated over 20 seeds for each method, horizon, and counterfactual setting.
We report the average
RMSE over
the 20 runs.

\subsection{Results}\label{subsec:results}

\begin{figure*}[t]
    \centering
    \includegraphics[width=\linewidth]{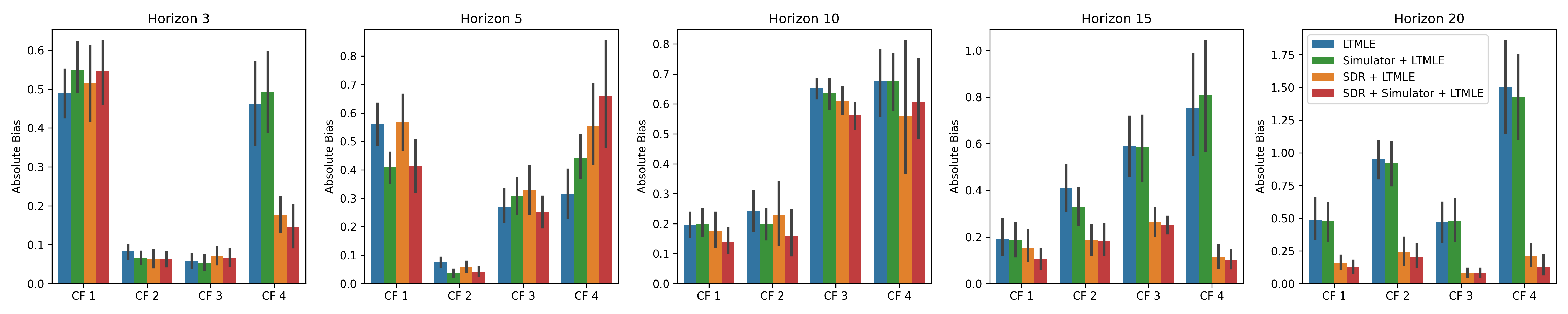}
    \vspace{-15pt}
    \caption{Ablation study on training components of $D^3$-Net.}
\label{fig:ablation_horizon_combined}
    \vspace{-\baselineskip}
\end{figure*}

\begin{figure*}[t]
    \centering
    \begin{minipage}{0.32\linewidth}
        \centering
        \includegraphics[width=\linewidth]{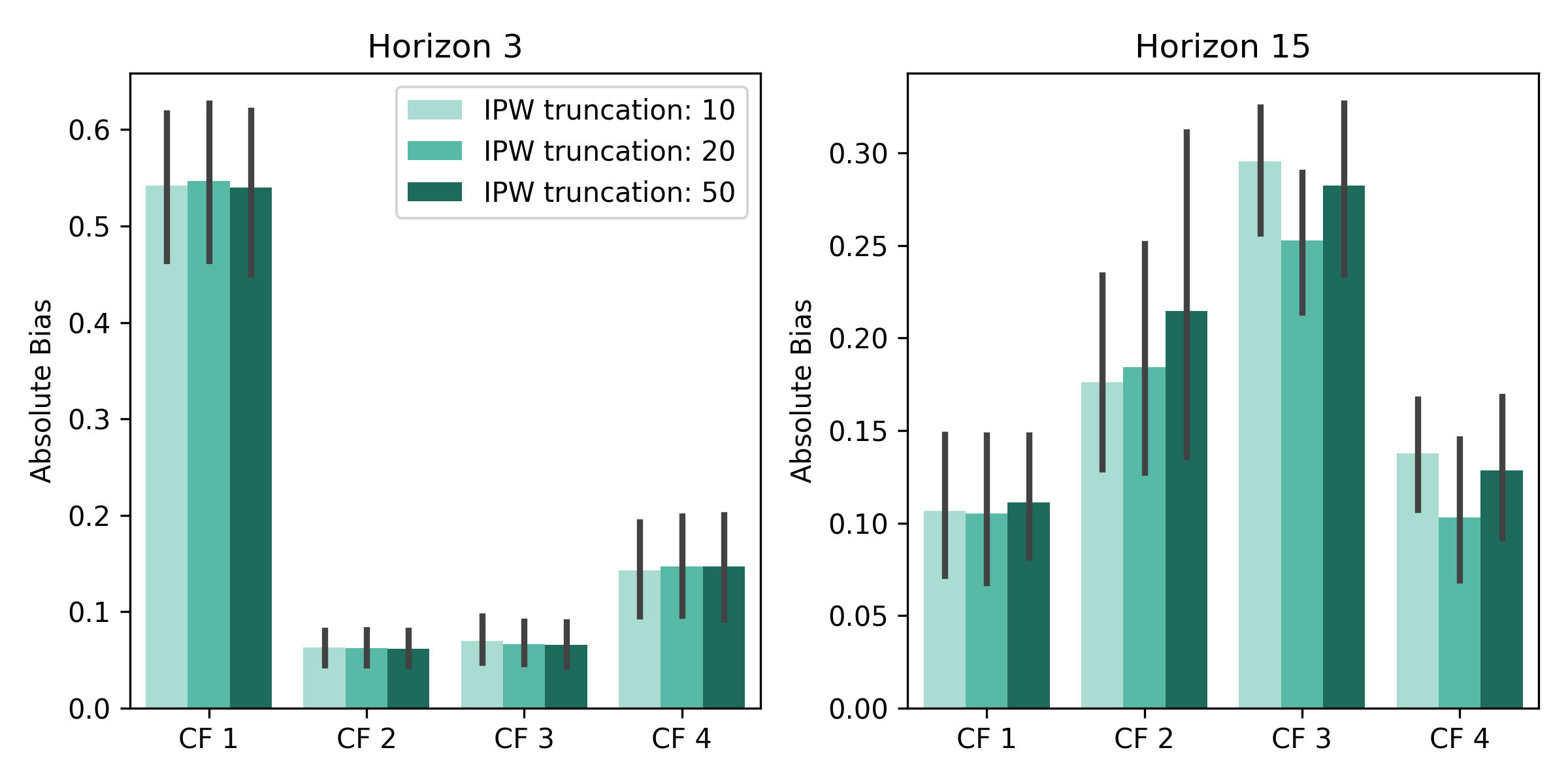}
        \vspace{-10pt}
        \caption{Varying IPW truncation thresholds in the SDR correction}
        \label{fig:ablation_trunc}
    \end{minipage}
    \hfill
    \begin{minipage}{0.32\linewidth}
        \centering
        \includegraphics[width=1\linewidth]{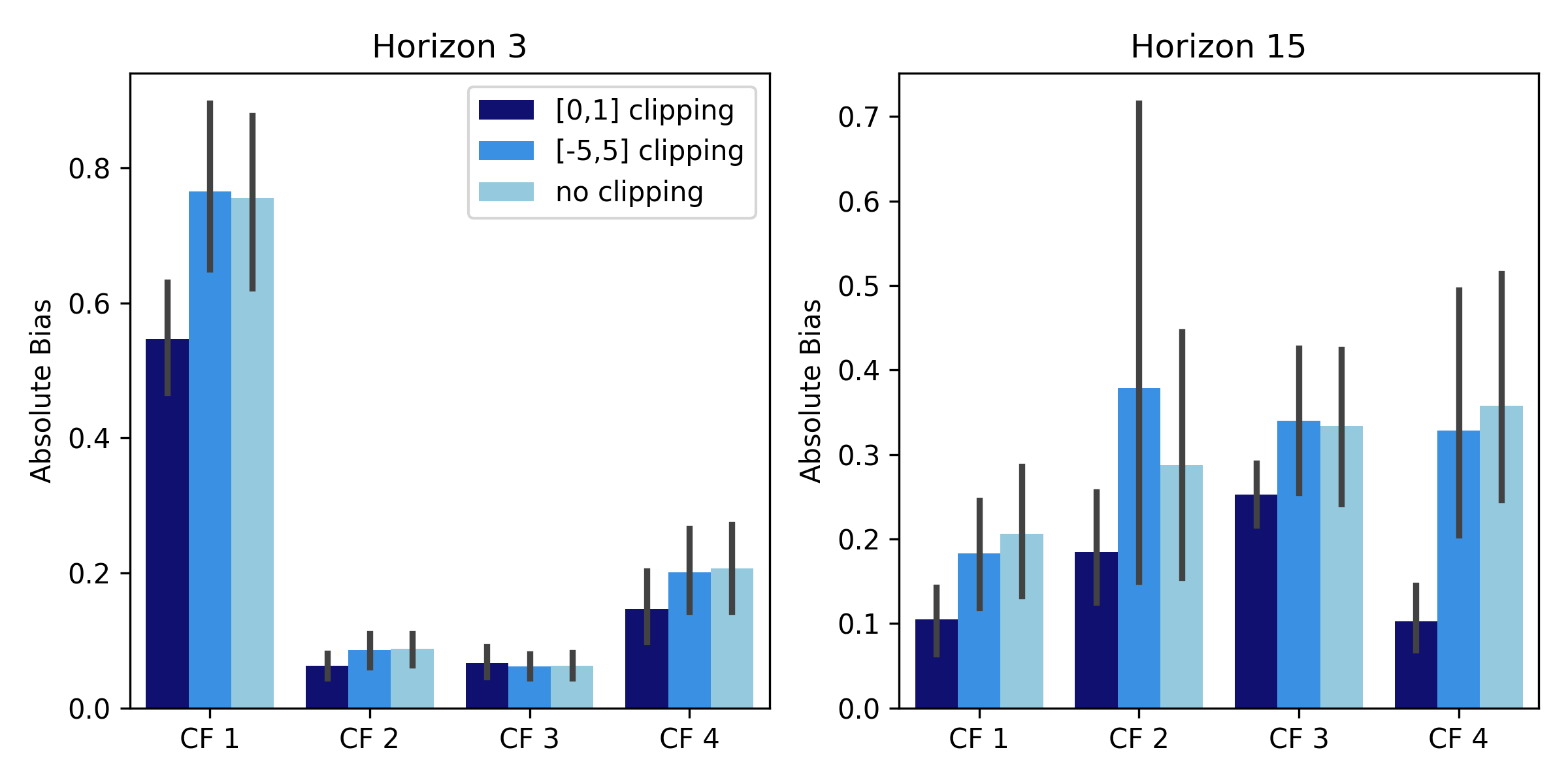}
        \vspace{-10pt}
        \caption{Varying clipping thresholds of SDR-augmented pseudo-outcome}
        \label{fig:ablation_clip}
    \end{minipage}
    \hfill
    \begin{minipage}{0.32\linewidth}
    \centering
    \includegraphics[width=1.05\linewidth]{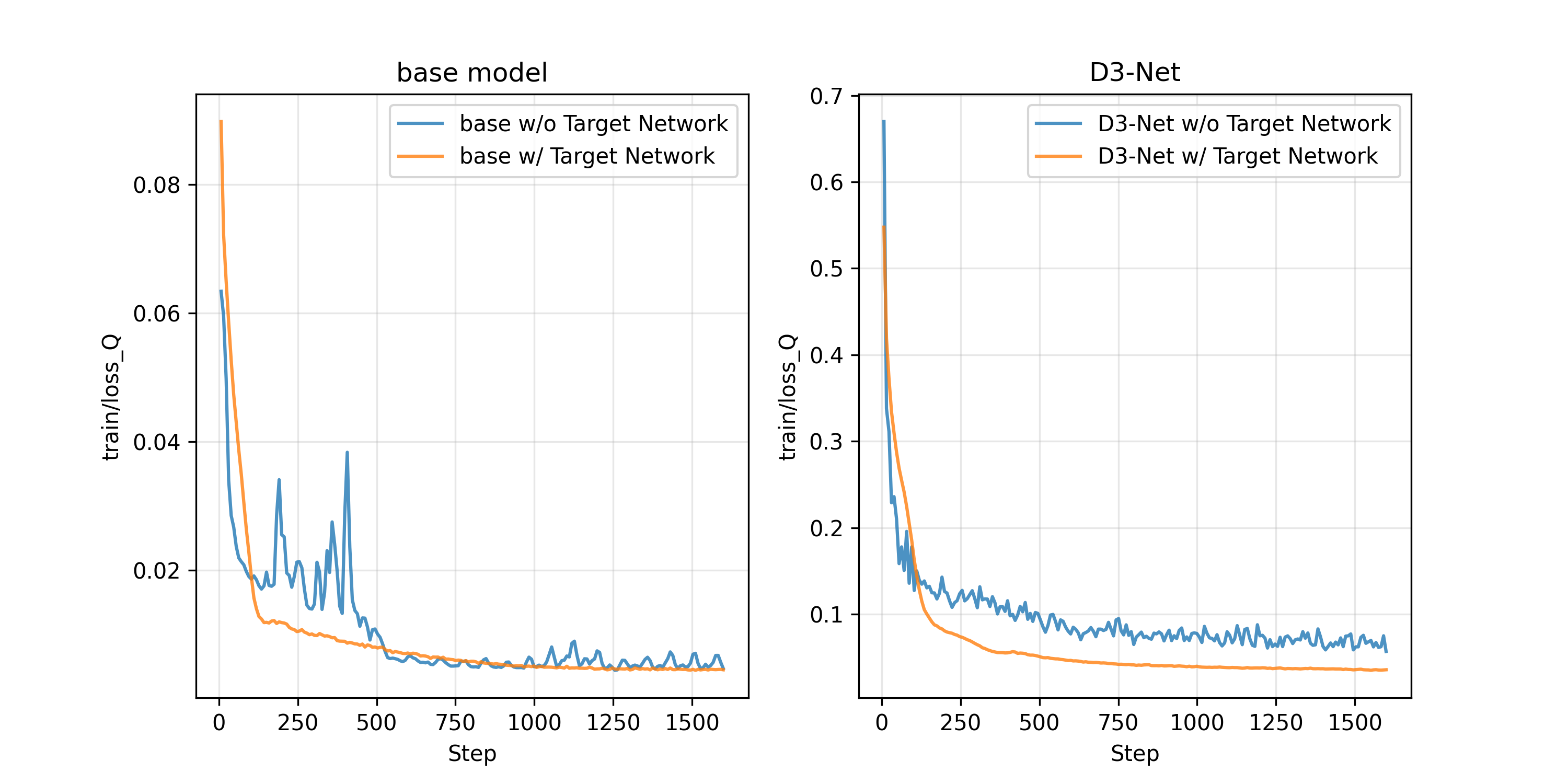}
    \vspace{-18pt}
    \caption{Training curve w/ and w/o Target Network}
    \label{fig:ablation_tar}
    \end{minipage}
    \vspace{-\baselineskip}
\end{figure*}

Figure~\ref{fig:main_result} summarizes performance by aggregating bias across counterfactual sequences. Across horizons and under both limited and expanded time-varying confounding, $D^3$-Net consistently ranks among the strongest-performing methods and generally achieves the lowest overall error, with its advantage becoming more pronounced at longer horizons. Tables~\ref{tab:results-tau15}, \ref{tab:results-tau20}, \ref{tab:results-tau10}, and \ref{tab:results-tau15-numz} report detailed results for each individual counterfactual sequence.

Beyond average performance, a key strength of $D^3$-Net is its \textbf{robustness across different settings}. While several baselines perform competitively for specific counterfactuals under limited confounding and short horizons, their performance degrades markedly as the horizon increases or the treatment–confounder feedback becomes stronger. In contrast, $D^3$-Net consistently attains the lowest or second-lowest RMSE across all horizons, confounding settings, and counterfactual sequences, demonstrating robust performance. This stability is particularly pronounced in the more challenging settings with longer horizons and expanded confounding.

Notably, despite having a similar expressiveness as DeepLTMLE because both models use a transformer backbone, $D^3$-Net yields substantial improvement. This gap highlights the benefit of training nuisances using the proposed SDR-debiased supervision together with auxiliary supervision. 
Overall, these results show that the advantages of $D^3$-Net are not confined to a single policy or regime, but persist across horizons, confounding strengths, and counterfactual sequences, highlighting its robustness to error propagation and treatment–confounder feedback.

\subsection{Ablation Study}\label{subsec:ablation}

\paragraph{SDR augmentation and auxiliary supervision} Figure~\ref{fig:ablation_horizon_combined} and Figure~\ref{fig:ablation} (left)
compare variants that selectively enable SDR and/or the simulator. Aggregated over counterfactual sequences, SDR is generally effective at reducing estimation bias, with the largest gains observed at longer horizons. This trend suggests that SDR helps mitigate the accumulation of prediction errors along the iterative G-computation recursion, thereby reducing error propagation. In contrast, auxiliary supervision alone provides only modest improvements over the baseline. However, when combined with SDR, it consistently yields further reductions in bias, indicating that the simulator contributes complementary information that enhances the effectiveness of the SDR correction.

\paragraph{Stabilization heuristics}
We adopt inverse-propensity truncation and clipping to stabilize the SDR-augmented pseudo-outcome. Figures~\ref{fig:ablation_trunc} and \ref{fig:ablation_clip} indicate that the model is 
insensitive to these heuristics at short horizons, but becomes increasingly sensitive as the horizon grows. For inverse-propensity truncation, strong truncation tends to increase bias, while weak truncation can increase variability through extreme weights. Pseudo-outcome clipping is generally beneficial, though its effect is likewise modest at short horizons. These findings are consistent with the overlap structure of the problem: shorter horizons exhibit better overlap and are therefore less sensitive to stabilization heuristics, whereas longer horizons are more susceptible to near-positivity violations and the resulting weight instability. 

\paragraph{Target Network}
Figure~\ref{fig:ablation_tar} shows that the target network substantially stabilizes the training dynamics, while Fig.~\ref{fig:ablation} (right)
shows no consistent effect on predictive performance. This suggests that the target network primarily serves as a stabilization mechanism rather than a direct source of performance gains.

\paragraph{Final LTMLE re-debiasing step} Following \citep{diaz2023nonparametric}, which studies the robustness of SDR and LTMLE under nuisance-model misspecification, we evaluate the impact of the final LTMLE targeting step by perturbing the learned outcome model. Figure~\ref{fig:ablation_ltmle} (left)
compares the raw SDR and LTMLE-updated estimates using the original nuisance models, while Figure~\ref{fig:ablation_ltmle} (right)
repeats the comparison after shifting the outcome predictions on the logit scale.\footnote{We perturb the outcome model as $Q_{\mathrm{pert}}=\mathrm{expit}(\mathrm{logit}(Q)+0.5)$.} Without perturbation, the two estimators exhibit similar bias across all horizons, indicating that the initial SDR estimate is already close to the target parameter. After perturbation, however, the LTMLE-updated estimator exhibits lower bias than the raw SDR estimate. This result suggests that the targeting step can partially correct errors in the outcome model and improve robustness to nuisance-model misspecification.

\paragraph{Runtime} D3-Net introduces only a small runtime overhead. The target network incurs negligible cost, the auxiliary supervision requires no additional forward passes, and the SDR objective reuses quantities already computed during ICE recursion. 
Figure~\ref{fig:runtime_real_world} (left)
confirms that SDR introduces negligible overhead across horizons.
 

\begin{figure}[t]
    \centering
    \begin{subfigure}{0.48\linewidth}
        \centering
        \includegraphics[width=\linewidth]{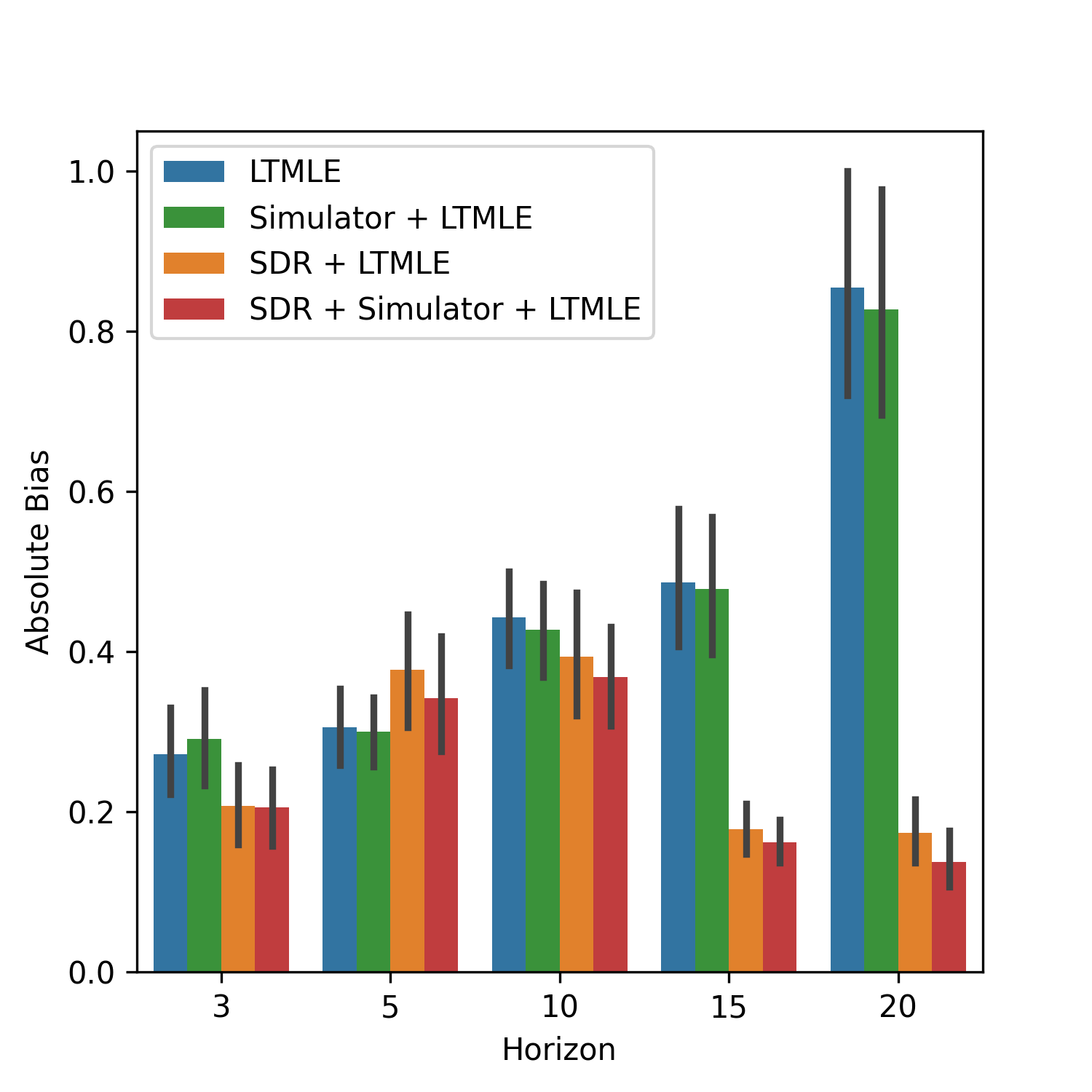}
        \label{fig:ablation:components}
    \end{subfigure}
    \hfill
    \begin{subfigure}{0.48\linewidth}
        \centering
        \includegraphics[width=\linewidth]{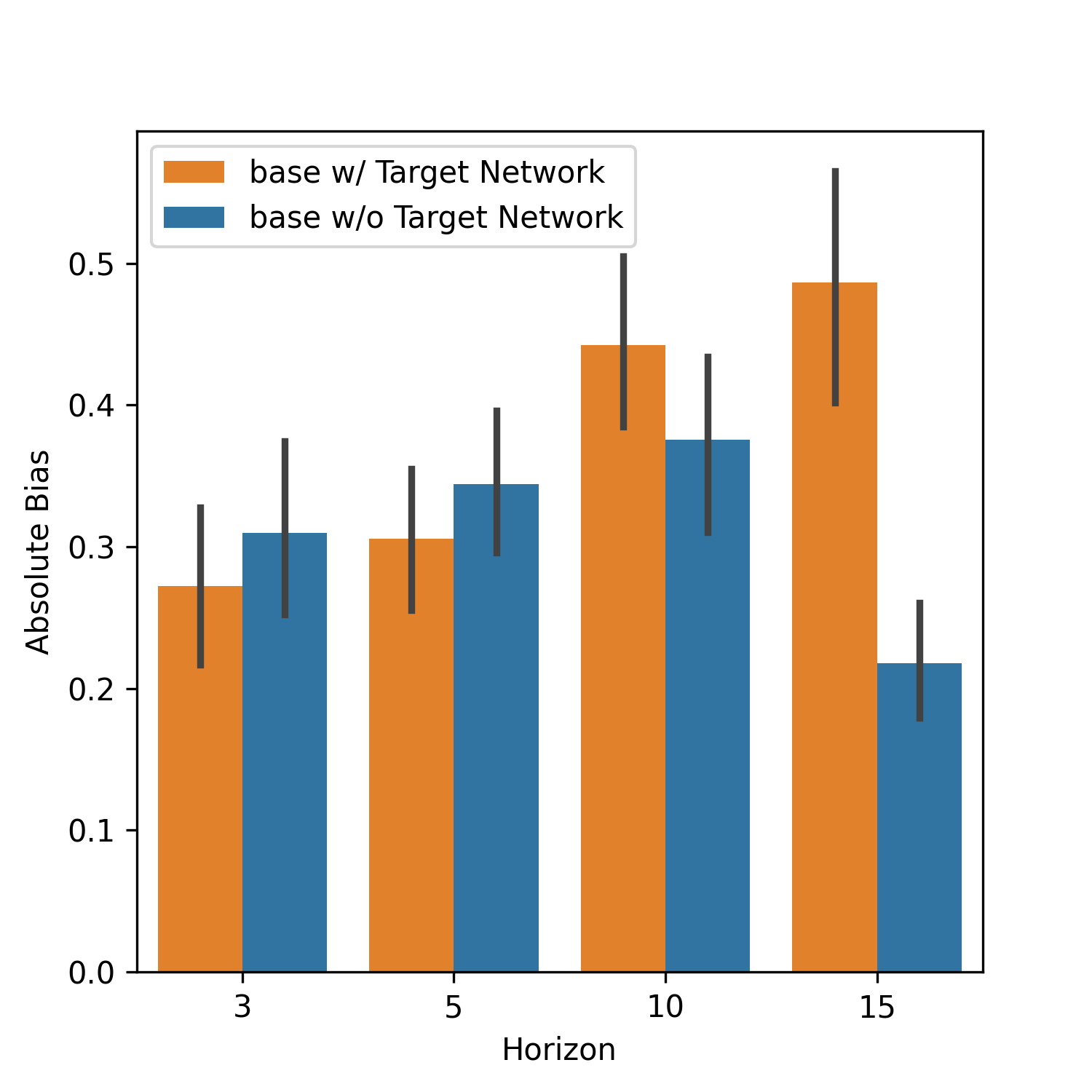}
        \label{fig:ablation:noTar}
    \end{subfigure}
    \vspace{-18pt}
    \caption{Ablation. Left: Contribution of SDR correction and Auxiliary Supervision. Right: Effect of the target network. Y-axis shows the mean absolute bias $\pm$ standard deviation.}
    \label{fig:ablation}
    \vspace{-\baselineskip}
\end{figure}

\begin{figure}[t]
    \centering
    \begin{subfigure}{0.49\linewidth}
        \centering
        \includegraphics[width=\linewidth]{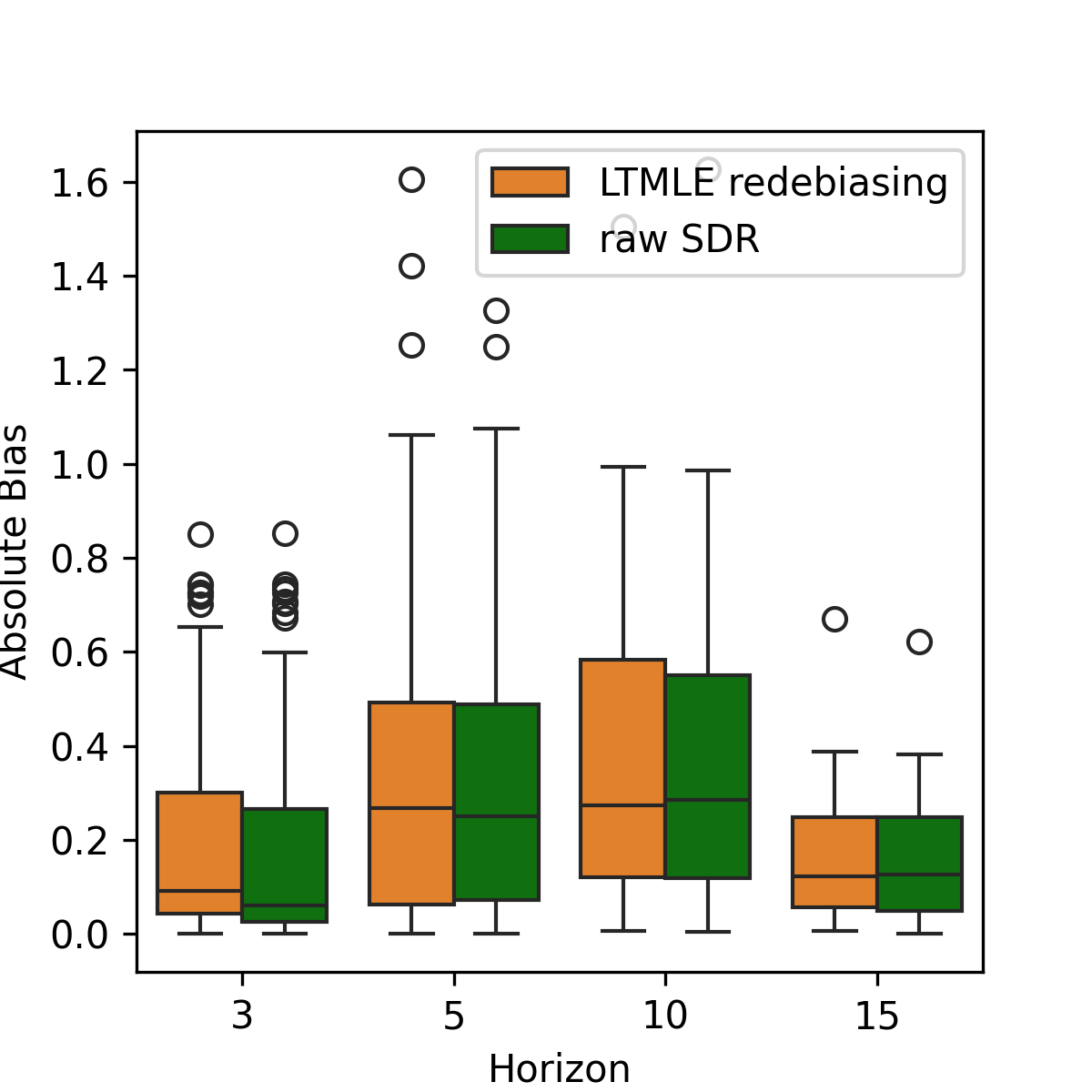}
        \label{fig:ltmle_original}
    \end{subfigure}
    \hfill
    \begin{subfigure}{0.49\linewidth}
        \centering
        \includegraphics[width=\linewidth]{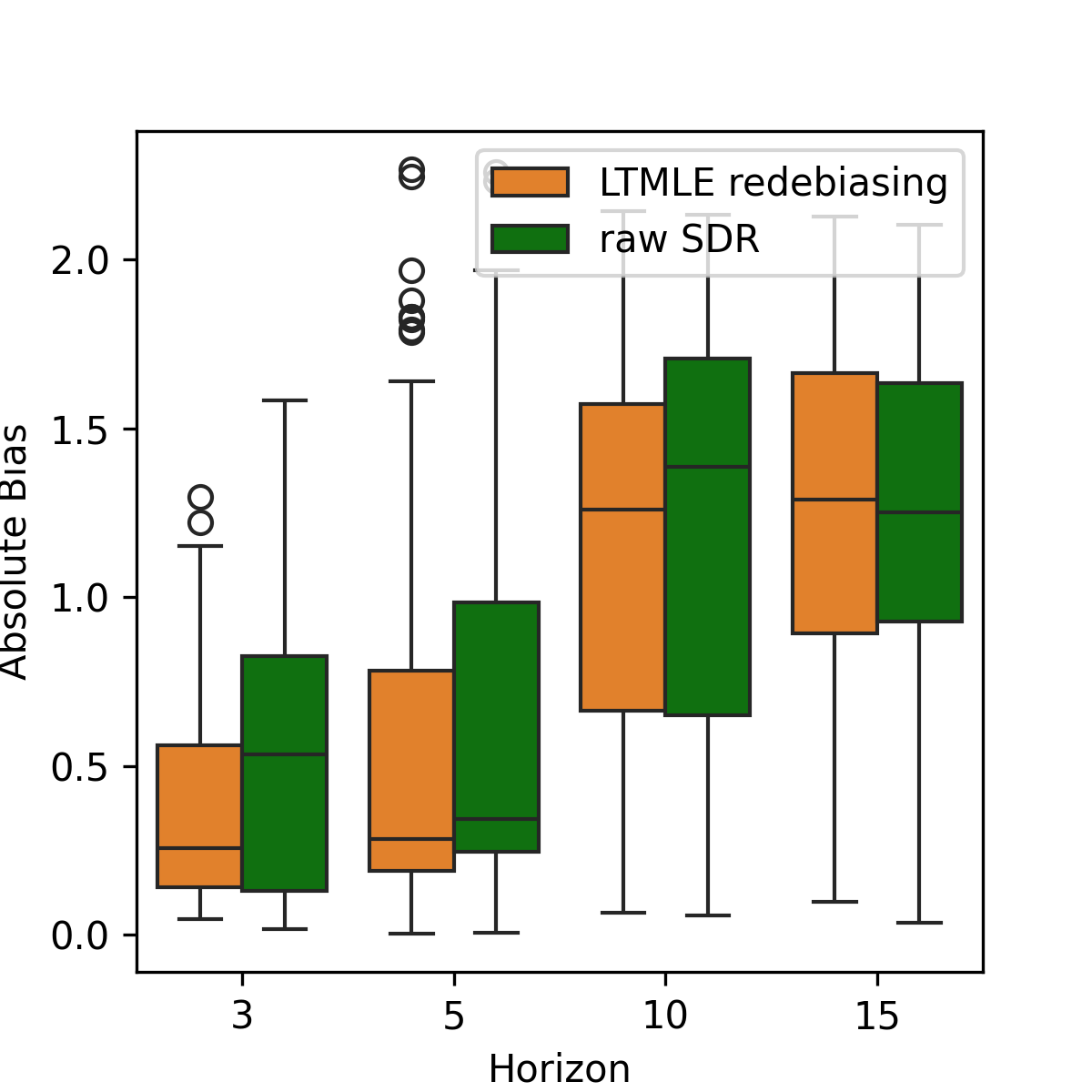}
        \label{fig:ltmle_q_corrupted}
    \end{subfigure}
    \vspace{-15pt}
    \caption{Comparison of raw SDR estimates and LTMLE re-debiasing before (left) and after (right) perturbing the learned (Q)-function. LTMLE re-debiasing is less sensitive to perturbations of the nuisance estimates.}
    \label{fig:ablation_ltmle}
    \vspace{-\baselineskip}
\end{figure}


\subsection{Real-world Case Study}\label{subsec:real-world}
We applied $D^3$-Net on an intensive care cohort derived from MIMIC-IV, focusing on patients with sepsis complicated by hypotension to study the effects of alternative vasopressor weaning strategies. The study population consists of 999 adult ICU patients who satisfied Sepsis-3 criteria, experienced hypotension (MAP $\leq 65$ mmHg), and received vasopressor within the first 24 hours after ICU admission. As the clinical endpoint, we considered serum lactate measured at 72 hours, which is widely used as an indicator of tissue hypoperfusion and disease severity in septic shock~\cite{nguyen2010early}.

We examined four treatment regimes corresponding to discontinuing vasopressors once MAP had been sustained above 65, 70, 75, or 80 mmHg for a 12-hour window.
Figure~\ref{fig:runtime_real_world} (right)
summarizes the resulting CAPO estimates of 20 Bootstrap experiments. Across these policies, we observe a pattern in which more aggressive MAP targets are associated with slightly higher lactate levels at 72 hours (higher lactate often implies worse outcome). This observation is consistent with prior randomized evidence showing no survival advantage from targeting higher MAP thresholds~\cite{asfar2014high}, and it is directionally aligned with current clinical guidelines that caution against unnecessary vasopressor exposure~\cite{evans2021surviving}.

\section{Conclusion}
We propose $D^3$-Net, a two-stage framework that integrates SDR-based training with LTMLE targeting, using a target network and simulator to stabilize and regularize deep ICE G-computation. Experiments demonstrate its robustness across settings.

\section{GenAI Disclosure}
We used generative AI tools only for language polishing and code debugging. All scientific ideas, methods, and experiments were developed and validated by the authors.

\section{Limitations and Ethical Considerations}
We use the publicly available, de-identified MIMIC-III and MIMIC-IV datasets in accordance with their data use policies. As with all observational studies, our findings rely on standard causal assumptions and may be affected by unmeasured confounding.

\begin{figure}[t]
    \centering
    \begin{subfigure}{0.48\linewidth}
        \includegraphics[width=\linewidth]{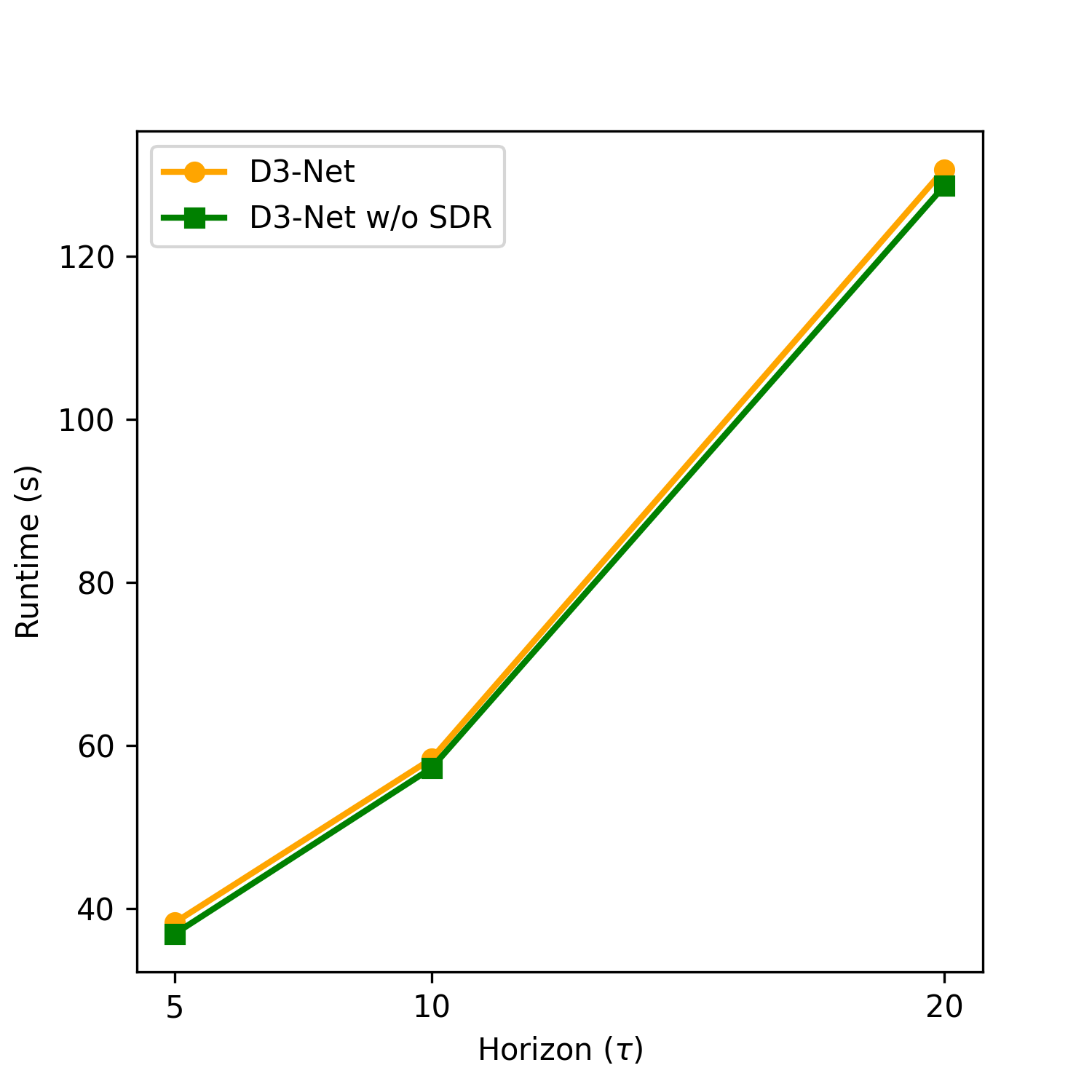}
        \label{fig:runtime}
    \end{subfigure}
    \begin{subfigure}{0.48\linewidth}
        \includegraphics[width=\linewidth]{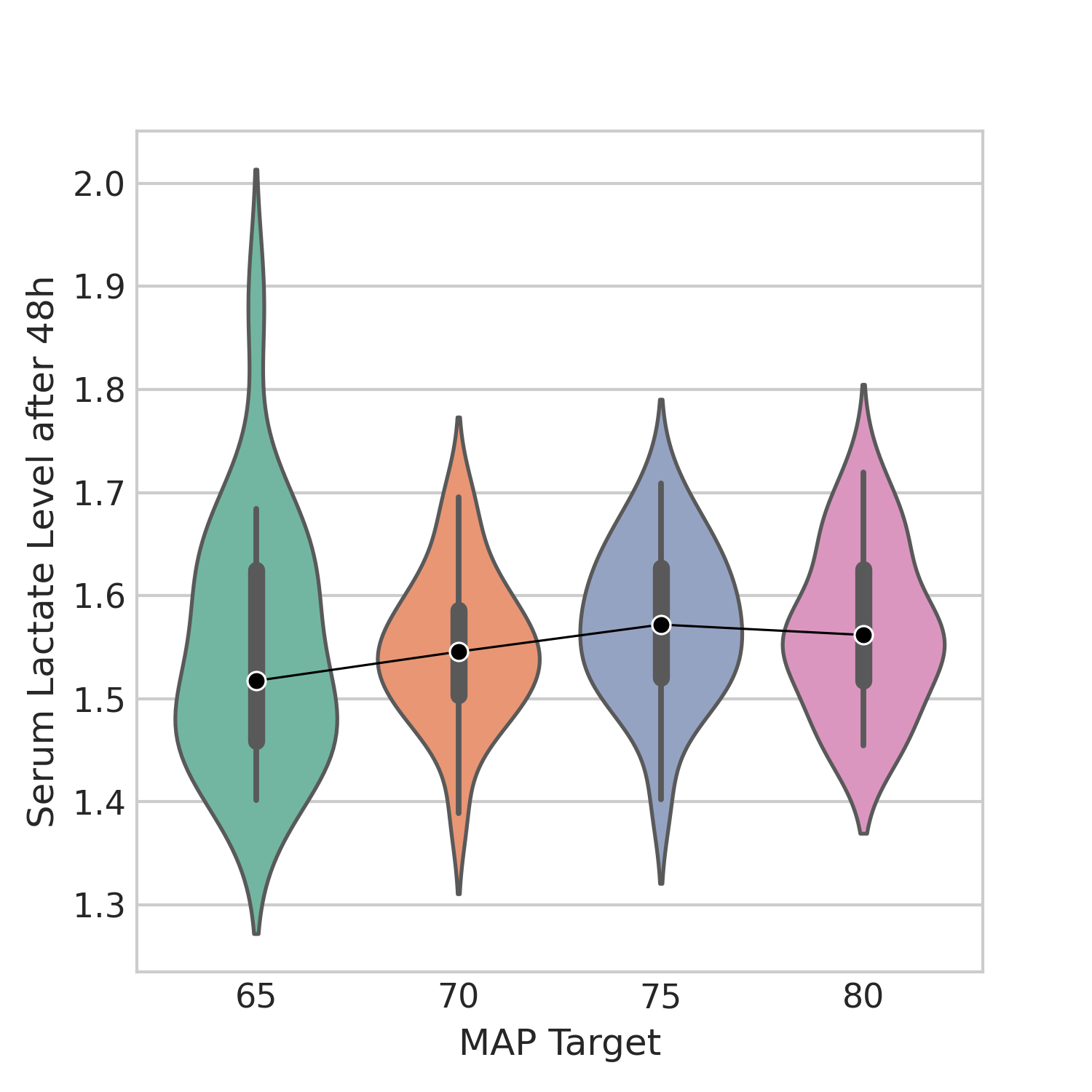}    
        \label{fig:sepsis_hypotension_lactate}
    \end{subfigure}
    \vspace{-15pt}
    \caption{Left: \emph{Runtime} under the same number of epochs. Right: \emph{Real-world case study}  --- distribution of CAPO estimates of 72-hour serum lactate across four MAP targets (65–80 mmHg), aggregated over 20 bootstrap experiments. 
    }
    \label{fig:runtime_real_world}
\end{figure}

\begin{acks}
We thank Xintong Li for her contributions to MIMIC-III data processing and the reproduction of the MIMIC-Extract data extraction pipeline, and thank Zhenxing Xu for his contribution to MIMIC-IV data processing. This work utilized computing resources provided by AWS. We would also like to acknowledge the support from NSF 2212175, NIH RF1AG072449, RF1AG084178, R01AG080991, R01AG080624, R01AG076448, R01AG076234, and R01NS140142.
\end{acks}

\bibliographystyle{ACM-Reference-Format}
\bibliography{reference}

@inproceedings{stooke2021decoupling,
  title={Decoupling representation learning from reinforcement learning},
  author={Stooke, Adam and Lee, Kimin and Abbeel, Pieter and Laskin, Michael},
  booktitle={International conference on machine learning},
  pages={9870--9879},
  year={2021},
  organization={PMLR}
}

@article{fang2023predictive,
  title={Predictive auxiliary objectives in deep rl mimic learning in the brain},
  author={Fang, Ching and Stachenfeld, Kimberly L},
  journal={arXiv preprint arXiv:2310.06089},
  year={2023}
}

@article{jaderberg2016reinforcement,
  title={Reinforcement learning with unsupervised auxiliary tasks},
  author={Jaderberg, Max and Mnih, Volodymyr and Czarnecki, Wojciech Marian and Schaul, Tom and Leibo, Joel Z and Silver, David and Kavukcuoglu, Koray},
  journal={arXiv preprint arXiv:1611.05397},
  year={2016}
}

@inproceedings{fu2019diagnosing,
  title={Diagnosing bottlenecks in deep q-learning algorithms},
  author={Fu, Justin and Kumar, Aviral and Soh, Matthew and Levine, Sergey},
  booktitle={International Conference on Machine Learning},
  pages={2021--2030},
  year={2019},
  organization={PMLR}
}

@article{petersen2014targeted,
  title={Targeted maximum likelihood estimation for dynamic and static longitudinal marginal structural working models},
  author={Petersen, Maya and Schwab, Joshua and Gruber, Susan and Blaser, Nello and Schomaker, Michael and van der Laan, Mark},
  journal={Journal of causal inference},
  volume={2},
  number={2},
  pages={147--185},
  year={2014},
  publisher={De Gruyter}
}

@article{tran2019double,
  title={Double robust efficient estimators of longitudinal treatment effects: comparative performance in simulations and a case study},
  author={Tran, Linh and Yiannoutsos, Constantin and Wools-Kaloustian, Kara and Siika, Abraham and Van Der Laan, Mark and Petersen, Maya},
  journal={The international journal of biostatistics},
  volume={15},
  number={2},
  pages={20170054},
  year={2019},
  publisher={De Gruyter}
}

@article{keil2014parametric,
  title={The parametric g-formula for time-to-event data: intuition and a worked example},
  author={Keil, Alexander P and Edwards, Jessie K and Richardson, David B and Naimi, Ashley I and Cole, Stephen R},
  journal={Epidemiology},
  volume={25},
  number={6},
  pages={889--897},
  year={2014},
  publisher={LWW}
}

@article{robins1986new,
  title={A new approach to causal inference in mortality studies with a sustained exposure period—application to control of the healthy worker survivor effect},
  author={Robins, James},
  journal={Mathematical modelling},
  volume={7},
  number={9-12},
  pages={1393--1512},
  year={1986},
  publisher={Elsevier}
}

@book{sutton1998reinforcement,
  title={Reinforcement learning: An introduction},
  author={Sutton, Richard S and Barto, Andrew G and others},
  volume={1},
  number={1},
  year={1998},
  publisher={MIT press Cambridge}
}

@article{bodory2022evaluating,
  title={Evaluating (weighted) dynamic treatment effects by double machine learning},
  author={Bodory, Hugo and Huber, Martin and Laff{\'e}rs, Luk{\'a}{\v{s}}},
  journal={The Econometrics Journal},
  volume={25},
  number={3},
  pages={628--648},
  year={2022},
  publisher={Oxford University Press}
}

@inproceedings{lewis2021double,
  title={Double/Debiased Machine Learning for Dynamic Treatment Effects.},
  author={Lewis, Greg and Syrgkanis, Vasilis},
  booktitle={NeurIPS},
  pages={22695--22707},
  year={2021}
}

@article{van2011targeted,
  title={Targeted minimum loss based estimation of causal effects of multiple time point interventions.},
  author={van der Laan, Mark J and Gruber, Susan},
  journal={International Journal of Biostatistics},
  volume={8},
  number={1},
  year={2011}
}

@article{lillicrap2015continuous,
  title={Continuous control with deep reinforcement learning},
  author={Lillicrap, Timothy P and Hunt, Jonathan J and Pritzel, Alexander and Heess, Nicolas and Erez, Tom and Tassa, Yuval and Silver, David and Wierstra, Daan},
  journal={arXiv preprint arXiv:1509.02971},
  year={2015}
}

@article{mnih2013playing,
  title={Playing atari with deep reinforcement learning},
  author={Mnih, Volodymyr and Kavukcuoglu, Koray and Silver, David and Graves, Alex and Antonoglou, Ioannis and Wierstra, Daan and Riedmiller, Martin},
  journal={arXiv preprint arXiv:1312.5602},
  year={2013}
}

@article{robins1994estimation,
  title={Estimation of regression coefficients when some regressors are not always observed},
  author={Robins, James M and Rotnitzky, Andrea and Zhao, Lue Ping},
  journal={Journal of the American statistical Association},
  volume={89},
  number={427},
  pages={846--866},
  year={1994},
  publisher={Taylor \& Francis}
}

@article{scharfstein1999adjusting,
  title={Adjusting for nonignorable drop-out using semiparametric nonresponse models},
  author={Scharfstein, Daniel O and Rotnitzky, Andrea and Robins, James M},
  journal={Journal of the American Statistical Association},
  volume={94},
  number={448},
  pages={1096--1120},
  year={1999},
  publisher={Taylor \& Francis}
}

@article{frauen2024model,
  title={Model-agnostic meta-learners for estimating heterogeneous treatment effects over time},
  author={Frauen, Dennis and Hess, Konstantin and Feuerriegel, Stefan},
  journal={arXiv preprint arXiv:2407.05287},
  year={2024}
}

@inproceedings{frauen2023estimating,
  title={Estimating average causal effects from patient trajectories},
  author={Frauen, Dennis and Hatt, Tobias and Melnychuk, Valentyn and Feuerriegel, Stefan},
  booktitle={Proceedings of the AAAI Conference on Artificial Intelligence},
  volume={37},
  number={6},
  pages={7586--7594},
  year={2023}
}

@inproceedings{shirakawa2024longitudinal,
  title={Longitudinal targeted minimum loss-based estimation with temporal-difference heterogeneous transformer},
  author={Shirakawa, Toru and Li, Yi and Wu, Yulun and Qiu, Sky and Li, Yuxuan and Zhao, Mingduo and Iso, Hiroyasu and Van Der Laan, Mark},
  booktitle={Proceedings of the 41st International Conference on Machine Learning},
  pages={45097--45113},
  year={2024}
}

@article{luedtke2017sequential,
  title={Sequential double robustness in right-censored longitudinal models},
  author={Luedtke, Alexander R and Sofrygin, Oleg and van der Laan, Mark J and Carone, Marco},
  journal={arXiv preprint arXiv:1705.02459},
  year={2017}
}

@article{diaz2023nonparametric,
  title={Nonparametric causal effects based on longitudinal modified treatment policies},
  author={D{\'\i}az, Iv{\'a}n and Williams, Nicholas and Hoffman, Katherine L and Schenck, Edward J},
  journal={Journal of the American Statistical Association},
  volume={118},
  number={542},
  pages={846--857},
  year={2023},
  publisher={Taylor \& Francis}
}

@article{evans2021surviving,
  title={Surviving sepsis campaign: international guidelines for management of sepsis and septic shock 2021},
  author={Evans, Laura and Rhodes, Andrew and Alhazzani, Waleed and Antonelli, Massimo and Coopersmith, Craig M and French, Craig and Machado, Fl{\'a}via R and Mcintyre, Lauralyn and Ostermann, Marlies and Prescott, Hallie C and others},
  journal={Critical care medicine},
  volume={49},
  number={11},
  pages={e1063--e1143},
  year={2021},
  publisher={LWW}
}

@article{asfar2014high,
  title={High versus low blood-pressure target in patients with septic shock},
  author={Asfar, Pierre and Meziani, Ferhat and Hamel, Jean-Fran{\c{c}}ois and Grelon, Fabien and Megarbane, Bruno and Anguel, Nadia and Mira, Jean-Paul and Dequin, Pierre-Fran{\c{c}}ois and Gergaud, Soizic and Weiss, Nicolas and others},
  journal={New England Journal of Medicine},
  volume={370},
  number={17},
  pages={1583--1593},
  year={2014},
  publisher={Mass Medical Soc}
}

@article{PhysioNet-mimiciii-1.4,
  author = {Johnson, Alistair and Pollard, Tom and Mark, Roger},
  title = {{MIMIC-III Clinical Database}},
  journal = {{PhysioNet}},
  year = {2016},
  month = sep,
  note = {Version 1.4},
  doi = {10.13026/C2XW26},
  url = {https://doi.org/10.13026/C2XW26}
}

@article{guo2024estimating,
  title={Estimating the treatment effect over time under general interference through deep learner integrated TMLE},
  author={Guo, Suhan and Shen, Furao and Li, Ni},
  journal={arXiv preprint arXiv:2412.04799},
  year={2024}
}

@article{rosenblum2010targeted,
  title={Targeted maximum likelihood estimation of the parameter of a marginal structural model},
  author={Rosenblum, Michael and Van Der Laan, Mark J},
  journal={The international journal of biostatistics},
  volume={6},
  number={2},
  pages={19},
  year={2010}
}

@inproceedings{cho2024kernel,
  title={Kernel Debiased Plug-in Estimation: Simultaneous, Automated Debiasing without Influence Functions for Many Target Parameters},
  author={Cho, Brian M and Mukhin, Yaroslav and Gan, Kyra and Malenica, Ivana},
  booktitle={International Conference on Machine Learning},
  pages={8534--8555},
  year={2024},
  organization={PMLR}
}

@article{stitelman2012general,
  title={A General Implementation of TMLE for Longitudinal Data Applied to Causal Inference in Survival Analysis},
  author={Stitelman, Ori M and De Gruttola, Victor and van der Laan, Mark J},
  journal={The International Journal of Biostatistics},
  volume={8},
  number={1},
  pages={1--39},
  year={2012},
  publisher={De Gruyter}
}

@incollection{robins2000marginal,
  title={Marginal structural models versus structural nested models as tools for causal inference},
  author={Robins, James M},
  booktitle={Statistical models in epidemiology, the environment, and clinical trials},
  pages={95--133},
  year={2000},
  publisher={Springer}
}

@article{bang2005doubly,
  title={Doubly robust estimation in missing data and causal inference models},
  author={Bang, Heejung and Robins, James M},
  journal={Biometrics},
  volume={61},
  number={4},
  pages={962--973},
  year={2005},
  publisher={Oxford University Press}
}

@article{snowden2011implementation,
  title={Implementation of G-computation on a simulated data set: demonstration of a causal inference technique},
  author={Snowden, Jonathan M and Rose, Sherri and Mortimer, Kathleen M},
  journal={American journal of epidemiology},
  volume={173},
  number={7},
  pages={731--738},
  year={2011},
  publisher={Oxford University Press}
}

@article{seedat2022continuous,
  title={Continuous-time modeling of counterfactual outcomes using neural controlled differential equations},
  author={Seedat, Nabeel and Imrie, Fergus and Bellot, Alexis and Qian, Zhaozhi and van der Schaar, Mihaela},
  journal={arXiv preprint arXiv:2206.08311},
  year={2022}
}

@inproceedings{bica2020estimating,
  title={Estimating counterfactual treatment outcomes over time through adversarially balanced representations},
  author={Bica, Ioana and Alaa, Ahmed M and Jordon, James and van der Schaar, Mihaela},
  booktitle={International Conference on Learning Representations},
year={2020}
}

@article{robins2008estimation,
  title={Estimation of the causal effects of time-varying exposures},
  author={Robins, James and Hernan, Miguel},
  journal={Chapman \& Hall/CRC Handbooks of Modern Statistical Methods},
  pages={553--599},
  year={2008},
  publisher={Chapman and Hall/CRC}
}

@misc{chernozhukov2018double,
  title={Double/debiased machine learning for treatment and structural parameters},
  author={Chernozhukov, Victor and Chetverikov, Denis and Demirer, Mert and Duflo, Esther and Hansen, Christian and Newey, Whitney and Robins, James},
  year={2018},
  publisher={Oxford University Press Oxford, UK}
}

@article{lim2018forecasting,
  title={Forecasting treatment responses over time using recurrent marginal structural networks},
  author={Lim, Bryan},
  journal={Advances in neural information processing systems},
  volume={31},
  year={2018}
}

@inproceedings{melnychuk2022causal,
  title={Causal transformer for estimating counterfactual outcomes},
  author={Melnychuk, Valentyn and Frauen, Dennis and Feuerriegel, Stefan},
  booktitle={International conference on machine learning},
  pages={15293--15329},
  year={2022},
  organization={PMLR}
}

@article{nguyen2010early,
  title={Early lactate clearance is associated with biomarkers of inflammation, coagulation, apoptosis, organ dysfunction and mortality in severe sepsis and septic shock},
  author={Nguyen, H Bryant and Loomba, Manisha and Yang, James J and Jacobsen, Gordon and Shah, Kant and Otero, Ronny M and Suarez, Arturo and Parekh, Hemal and Jaehne, Anja and Rivers, Emanuel P},
  journal={Journal of inflammation},
  volume={7},
  number={1},
  pages={6},
  year={2010},
  publisher={Springer}
}

@article{taubman2009intervening,
  title={Intervening on risk factors for coronary heart disease: an application of the parametric g-formula},
  author={Taubman, Sarah L and Robins, James M and Mittleman, Murray A and Hern{\'a}n, Miguel A},
  journal={International journal of epidemiology},
  volume={38},
  number={6},
  pages={1599--1611},
  year={2009},
  publisher={Oxford University Press}
}

@inproceedings{li2021g,
  title={G-net: a recurrent network approach to g-computation for counterfactual prediction under a dynamic treatment regime},
  author={Li, Rui and Hu, Stephanie and Lu, Mingyu and Utsumi, Yuria and Chakraborty, Prithwish and Sow, Daby M and Madan, Piyush and Li, Jun and Ghalwash, Mohamed and Shahn, Zach and others},
  booktitle={Machine Learning for Health},
  pages={282--299},
  year={2021},
  organization={PMLR}
}

@inproceedings{bibaut2019more,
  title={More efficient off-policy evaluation through regularized targeted learning},
  author={Bibaut, Aurelien and Malenica, Ivana and Vlassis, Nikos and Van Der Laan, Mark},
  booktitle={International Conference on Machine Learning},
  pages={654--663},
  year={2019},
  organization={PMLR}
}

@inproceedings{hess2024igc,
  title={Igc-net for conditional average potential outcome estimation over time},
  author={Hess, Konstantin and Frauen, Dennis and Melnychuk, Valentyn and Feuerriegel, Stefan},
  booktitle={The Fourteenth International Conference on Learning Representations},
  year={2024}
}

@article{hernan2006estimating,
  title={Estimating causal effects from epidemiological data},
  author={Hern{\'a}n, Miguel A and Robins, James M},
  journal={Journal of Epidemiology \& Community Health},
  volume={60},
  number={7},
  pages={578--586},
  year={2006},
  publisher={BMJ Publishing Group Ltd}
}

@article{cole2008constructing,
  title={Constructing inverse probability weights for marginal structural models},
  author={Cole, Stephen R and Hern{\'a}n, Miguel A},
  journal={American journal of epidemiology},
  volume={168},
  number={6},
  pages={656--664},
  year={2008},
  publisher={Oxford University Press}
}

@article{rosenbaum1983central,
  title={The central role of the propensity score in observational studies for causal effects},
  author={Rosenbaum, Paul R and Rubin, Donald B},
  journal={Biometrika},
  volume={70},
  number={1},
  pages={41--55},
  year={1983},
  publisher={Oxford University Press}
}

@article{rosenbaum1984reducing,
  title={Reducing bias in observational studies using subclassification on the propensity score},
  author={Rosenbaum, Paul R and Rubin, Donald B},
  journal={Journal of the American statistical Association},
  volume={79},
  number={387},
  pages={516--524},
  year={1984},
  publisher={Taylor \& Francis}
}

@article{xu2010use,
  title={Use of stabilized inverse propensity scores as weights to directly estimate relative risk and its confidence intervals},
  author={Xu, Stanley and Ross, Colleen and Raebel, Marsha A and Shetterly, Susan and Blanchette, Christopher and Smith, David},
  journal={Value in Health},
  volume={13},
  number={2},
  pages={273--277},
  year={2010},
  publisher={Wiley Online Library}
}

@article{chesnaye2022introduction,
  title={An introduction to inverse probability of treatment weighting in observational research},
  author={Chesnaye, Nicholas C and Stel, Vianda S and Tripepi, Giovanni and Dekker, Friedo W and Fu, Edouard L and Zoccali, Carmine and Jager, Kitty J},
  journal={Clinical kidney journal},
  volume={15},
  number={1},
  pages={14--20},
  year={2022},
  publisher={Oxford University Press}
}

@article{austin2011introduction,
  title={An introduction to propensity score methods for reducing the effects of confounding in observational studies},
  author={Austin, Peter C},
  journal={Multivariate behavioral research},
  volume={46},
  number={3},
  pages={399--424},
  year={2011},
  publisher={Taylor \& Francis}
}

@article{stuart2010matching,
  title={Matching methods for causal inference: A review and a look forward},
  author={Stuart, Elizabeth A},
  journal={Statistical science: a review journal of the Institute of Mathematical Statistics},
  volume={25},
  number={1},
  pages={1},
  year={2010}
}

@article{mahmood2014weighted,
  title={Weighted importance sampling for off-policy learning with linear function approximation},
  author={Mahmood, A Rupam and Van Hasselt, Hado P and Sutton, Richard S},
  journal={Advances in neural information processing systems},
  volume={27},
  year={2014}
}

@article{petersen2012diagnosing,
  title={Diagnosing and responding to violations in the positivity assumption},
  author={Petersen, Maya L and Porter, Kristin E and Gruber, Susan and Wang, Yue and Van Der Laan, Mark J},
  journal={Statistical methods in medical research},
  volume={21},
  number={1},
  pages={31--54},
  year={2012},
  publisher={SAGE Publications Sage UK: London, England}
}

@article{chiu2023evaluating,
  title={Evaluating model specification when using the parametric g-formula in the presence of censoring},
  author={Chiu, Yu-Han and Wen, Lan and McGrath, Sean and Logan, Roger and Dahabreh, Issa J and Hern{\'a}n, Miguel A},
  journal={American journal of epidemiology},
  volume={192},
  number={11},
  pages={1887--1895},
  year={2023},
  publisher={Oxford University Press}
}

@article{mcgrath2020gformula,
  title={gfoRmula: an R package for estimating the effects of sustained treatment strategies via the parametric g-formula},
  author={McGrath, Sean and Lin, Victoria and Zhang, Zilu and Petito, Lucia C and Logan, Roger W and Hern{\'a}n, Miguel A and Young, Jessica G},
  journal={Patterns},
  volume={1},
  number={3},
  year={2020},
  publisher={Elsevier}
}

@inproceedings{oprescu2025gst,
  title={GST-UNet: A Neural Framework for Spatiotemporal Causal Inference with Time-Varying Confounding},
  author={Oprescu, Miruna and Park, David Keetae and Luo, Xihaier and Yoo, Shinjae and Kallus, Nathan},
  booktitle={The Thirty-ninth Annual Conference on Neural Information Processing Systems},
  year={2025}
}

@article{liu2025bayesian,
  title={A Bayesian Approach to the G-Formula via Iterative Conditional Regression},
  author={Liu, Ruyi and Hu, Liangyuan and Perry Wilson, Francis and Warren, Joshua L and Li, Fan},
  journal={Statistics in Medicine},
  volume={44},
  number={13-14},
  pages={e70123},
  year={2025},
  publisher={Wiley Online Library}
}

@article{rein2024deep,
  title={Deep Learning Methods for the Noniterative Conditional Expectation G-Formula for Causal Inference from Complex Observational Data},
  author={Rein, Sophia M and Li, Jing and Hernan, Miguel and Beam, Andrew},
  journal={arXiv preprint arXiv:2410.21531},
  year={2024}
}

@inproceedings{thomas2016data,
  title={Data-efficient off-policy policy evaluation for reinforcement learning},
  author={Thomas, Philip and Brunskill, Emma},
  booktitle={International conference on machine learning},
  pages={2139--2148},
  year={2016},
  organization={PMLR}
}

@inproceedings{hanna2019importance,
  title={Importance sampling policy evaluation with an estimated behavior policy},
  author={Hanna, Josiah and Niekum, Scott and Stone, Peter},
  booktitle={International Conference on Machine Learning},
  pages={2605--2613},
  year={2019},
  organization={PMLR}
}

@article{schlegel2019importance,
  title={Importance resampling for off-policy prediction},
  author={Schlegel, Matthew and Chung, Wesley and Graves, Daniel and Qian, Jian and White, Martha},
  journal={Advances in Neural Information Processing Systems},
  volume={32},
  year={2019}
}

@article{liu2018breaking,
  title={Breaking the curse of horizon: Infinite-horizon off-policy estimation},
  author={Liu, Qiang and Li, Lihong and Tang, Ziyang and Zhou, Dengyong},
  journal={Advances in neural information processing systems},
  volume={31},
  year={2018}
}

@inproceedings{le2019batch,
  title={Batch policy learning under constraints},
  author={Le, Hoang and Voloshin, Cameron and Yue, Yisong},
  booktitle={International Conference on Machine Learning},
  pages={3703--3712},
  year={2019},
  organization={PMLR}
}

@inproceedings{duan2020minimax,
  title={Minimax-optimal off-policy evaluation with linear function approximation},
  author={Duan, Yaqi and Jia, Zeyu and Wang, Mengdi},
  booktitle={International Conference on Machine Learning},
  pages={2701--2709},
  year={2020},
  organization={PMLR}
}

@misc{voloshin2021empiricalstudyoffpolicypolicy,
      title={Empirical Study of Off-Policy Policy Evaluation for Reinforcement Learning}, 
      author={Cameron Voloshin and Hoang M. Le and Nan Jiang and Yisong Yue},
      year={2021},
      eprint={1911.06854},
      archivePrefix={arXiv},
      primaryClass={cs.LG},
      url={https://arxiv.org/abs/1911.06854}, 
}

@InProceedings{pmlr-v48-jiang16,
  title = 	 {Doubly Robust Off-policy Value Evaluation for Reinforcement Learning},
  author = 	 {Jiang, Nan and Li, Lihong},
  booktitle = 	 {Proceedings of The 33rd International Conference on Machine Learning},
  pages = 	 {652--661},
  year = 	 {2016},
  editor = 	 {Balcan, Maria Florina and Weinberger, Kilian Q.},
  volume = 	 {48},
  series = 	 {Proceedings of Machine Learning Research},
  address = 	 {New York, New York, USA},
  month = 	 {20--22 Jun},
  publisher =    {PMLR},
  url = 	 {https://proceedings.mlr.press/v48/jiang16.html}
}

@InProceedings{pmlr-v48-thomasa16,
  title = 	 {Data-Efficient Off-Policy Policy Evaluation for Reinforcement Learning},
  author = 	 {Thomas, Philip and Brunskill, Emma},
  booktitle = 	 {Proceedings of The 33rd International Conference on Machine Learning},
  pages = 	 {2139--2148},
  year = 	 {2016},
  editor = 	 {Balcan, Maria Florina and Weinberger, Kilian Q.},
  volume = 	 {48},
  series = 	 {Proceedings of Machine Learning Research},
  address = 	 {New York, New York, USA},
  month = 	 {20--22 Jun},
  publisher =    {PMLR},
  url = 	 {https://proceedings.mlr.press/v48/thomasa16.html}
}

@article{JMLR:v21:19-827,
  author  = {Nathan Kallus and Masatoshi Uehara},
  title   = {Double Reinforcement Learning for Efficient Off-Policy Evaluation in Markov Decision Processes},
  journal = {Journal of Machine Learning Research},
  year    = {2020},
  volume  = {21},
  number  = {167},
  pages   = {1--63},
  url     = {http://jmlr.org/papers/v21/19-827.html}
}

@article{DBLP:journals/corr/abs-1103-4601,
  author       = {Miroslav Dud{\'{\i}}k and
                  John Langford and
                  Lihong Li},
  title        = {Doubly Robust Policy Evaluation and Learning},
  journal      = {CoRR},
  volume       = {abs/1103.4601},
  year         = {2011},
  url          = {http://arxiv.org/abs/1103.4601},
  eprinttype    = {arXiv},
  eprint       = {1103.4601},
  bibsource    = {dblp computer science bibliography, https://dblp.org}
}

@inproceedings{
li2025targeted,
title={Targeted Maximum Likelihood Learning: An Optimization Perspective},
author={Diyang Li and Kyra Gan},
booktitle={The Thirty-ninth Annual Conference on Neural Information Processing Systems},
year={2025},
url={https://openreview.net/forum?id=n63KgrgVHG}
}

@inproceedings{Wang_2020, series={ACM CHIL ’20},
   title={MIMIC-Extract: a data extraction, preprocessing, and representation pipeline for MIMIC-III},
   url={http://dx.doi.org/10.1145/3368555.3384469},
   DOI={10.1145/3368555.3384469},
   booktitle={Proceedings of the ACM Conference on Health, Inference, and Learning},
   publisher={ACM},
   author={Wang, Shirly and McDermott, Matthew B. A. and Chauhan, Geeticka and Ghassemi, Marzyeh and Hughes, Michael C. and Naumann, Tristan},
   year={2020},
   month=apr, pages={222–235},
   collection={ACM CHIL ’20} }

@article{lendle2017ltmle,
  title={ltmle: an R package implementing targeted minimum loss-based estimation for longitudinal data},
  author={Lendle, Samuel D and Schwab, Joshua and Petersen, Maya L and van der Laan, Mark J},
  journal={Journal of Statistical Software},
  volume={81},
  pages={1--21},
  year={2017}
}

\appendix
\setcounter{figure}{0}
\renewcommand{\thefigure}{S\arabic{figure}}
\setcounter{table}{0}
\renewcommand{\thetable}{S\arabic{table}}

\section{Proof}\label{appx:first_vs_second_order}
\LemmaFirstSecondOrder*

\begin{proof}
Consider a CF sequence $\ba=(a_1,\dots,a_\tau)$.
Recall the target regression at time $t$:
$
Q_t^0(A_t,H_t):=\mathbb{E}\!\left[\,Q_{t+1}^0(a_{t+1},H_{t+1})\mid A_t,H_t\,\right].
$
Throughout, positivity means $G_k^0(H_k),\hat G_k^{\mathrm{SDR}}(H_k)\in[\delta,1-\delta]$ a.s.\ for all $k\in[\tau]$, so all density ratios below are uniformly bounded.

\paragraph{Step 1: a generic decomposition.}
For any training target $T_{t+1}$ (measurable w.r.t.\ $(H_{t+1},A_{t+1})$), define its conditional mean:
$
m_t^{T}(A_t,H_t):=\mathbb{E}[T_{t+1}\mid A_t, H_t].
$
Then by the triangle inequality,
\begin{equation}\label{eq:decomp_generic}
\|\hat Q_t-Q_t^0\|_2
\le
\underbrace{\|\hat Q_t-m_t^{T}\|_2}_{:=\xi_t}
+
\|m_t^{T}-Q_t^0\|_2.
\end{equation}
We apply~\eqref{eq:decomp_generic} to the ICE target and the SDR target separately.

\paragraph{Step 2: ICE target yields a first-order bias term.}
For ICE, the training target is
$
T^{\mathrm{ICE}}_{t+1}:=\hat Q_{t+1}^{\mathrm{ICE}}(a_{t+1},H_{t+1}).
$
Hence,
$
m_t^{\mathrm{ICE}}-Q_t^0
=
\mathbb{E}\!\left[\hat Q_{t+1}^{\mathrm{ICE}}(a_{t+1},H_{t+1})-Q_{t+1}^0(a_{t+1},H_{t+1})\mid A_t,H_t\right].
$
By Jensen (equivalently, $\|\mathbb{E}[U\mid A_t,H_t]\|_2\le \|U\|_2$),
\[
\|m_t^{\mathrm{ICE}}-Q_t^0\|_2
\le
\|\hat Q_{t+1}^{\mathrm{ICE}}(a_{t+1},H_{t+1})-Q_{t+1}^0(a_{t+1},H_{t+1})\|_2.
\]
Combining this with~\eqref{eq:decomp_generic} proves the ICE bound.

\paragraph{Step 3: SDR target yields a second-order bias term.}
For SDR, the training target is
$
T^{\mathrm{SDR}}_{t+1}:=Q^\dagger_{t+1}(a_{t+1},H_{t+1}),
$
where $Q^\dagger$ is constructed via~\eqref{eq:sdr_target} using nuisance estimates
$(\hat Q^{\mathrm{SDR}},\hat G^{\mathrm{SDR}})$.
Let
$
m_t^{\mathrm{SDR}}(A_t,H_t):=\mathbb{E}[Q^\dagger_{t+1}(a_{t+1},H_{t+1})\mid A_t,H_t].
$
We will bound $\|m_t^{\mathrm{SDR}}-Q_t^0\|_2$.

Define the true and estimated density ratios for the sequence $\ba$:
\begin{align*}
w_k^0(A_k,H_k)
&:=
\frac{\mathds{1}(A_k=a_k)}{A_k G_k^0(H_k)+(1-A_k)(1-G_k^0(H_k))},\\
\hat w_k(A_k,H_k)
&:=
\frac{\mathds{1}(A_k=a_k)}{A_k \hat G_k^{\mathrm{SDR}}(H_k)+(1-A_k)(1-\hat G_k^{\mathrm{SDR}}(H_k))}.
\end{align*}

Under positivity, $\|\hat w_k\|_\infty\vee\| w_k^0\|_\infty\le \delta^{-1}$ and
\begin{equation}\label{eq:eta_lip}
\|\hat w_k-w_k^0\|_2 \lesssim \|\hat G_k^{\mathrm{SDR}}-G_k^0\|_2.
\end{equation}

Now consider the difference between the SDR pseudo-outcome and its population counterpart. Lemma 3 in ~\citet{diaz2023nonparametric} implies that, for each fixed $t$,
\begin{equation}\label{eq:sdr_bias_key}
m_t^{\mathrm{SDR}}-Q_t^0
=
\sum_{k=t+1}^{\tau}
\mathbb{E}\!\left[
W_{t,k}\,(\hat w_k-w_k^0)\,\bigl(\hat Q_k^{\mathrm{SDR}}-Q_k^\dagger\bigr)
\ \middle|\ A_t,H_t
\right],
\end{equation}
where $W_{t,k}$ is a product of (true and/or estimated) weight factors over times $t+1,\dots,k-1$.
By positivity, $\|W_{t,k}\|_\infty \le C'_{t,k}$ for constants $C'_{t,k}$ depending only on $(\delta,\tau)$.

Apply Jensen and Cauchy--Schwarz to~\eqref{eq:sdr_bias_key}:
\begin{align*}
\|m_t^{\mathrm{SDR}}-Q_t^0\|_2
&\le
\sum_{k=t+1}^{\tau}
\left\|
E\!\left[
W_{t,k}\,(\hat w_k- w_k^0)\,(\hat Q_k^{\mathrm{SDR}}-Q_k^\dagger)
\ \middle|\ A_t,H_t
\right]
\right\|_2 \\
&\le
\sum_{k=t+1}^{\tau}
\|W_{t,k}\|_\infty\,
\|(\hat w_k- w_k^0)\,(\hat Q_k^{\mathrm{SDR}}-Q_k^\dagger)\|_2 \\
&\le
\sum_{k=t+1}^{\tau}
C'_{t,k}\,\|\hat w_k- w_k^0\|_2\,\|\hat Q_k^{\mathrm{SDR}}-Q_k^\dagger\|_2.
\end{align*}
Invoking~\eqref{eq:eta_lip} yields
\[
\|m_t^{\mathrm{SDR}}-Q_t^0\|_2
\le
\sum_{k=t+1}^{\tau}
C_{t,k}\,
\|\hat G_k^{\mathrm{SDR}}-G_k^0\|_2\,\|\hat Q_k^{\mathrm{SDR}}-Q_k^\dagger\|_2,
\]
for constants $C_{t,k}$ depending only on $(\delta,\tau)$.
Finally, combine this bound with the generic decomposition~\eqref{eq:decomp_generic} (with $T=T^{\mathrm{SDR}}$),
where $\xi_t=\|\hat Q_t^{\mathrm{SDR}}-m_t^{\mathrm{SDR}}\|_2$, to obtain the SDR inequality in the theorem.
\end{proof}

\section{LTMLE}\label{appx:ltmle}
We adopted the one-step LTMLE algorithm~\cite{van2011targeted} and added an L1 regularizer for the targeting update~\cite{bibaut2019more}, as detailed in Algorithm~\ref{algo:ltmle}. 
\begin{algorithm}[t]
\caption{LTMLE}
\begin{algorithmic}[1]\label{algo:ltmle}
    \STATE \textbf{Input}: Dataset $\{O_i\}_{i=1}^n$; Initial estimates of $Q_t$ and $G_t$ for $t = 1, \dots, \tau$; L1 regularization coefficient: $\lambda$.
    \STATE \textbf{Output}: CAPO estimate $\hat{\psi}_n = \mathbb{P}_nQ_{1,\epsilon_1}(a_1, H_1)$
    \STATE Initialize $Q_{\tau+1} = Y$, $\epsilon_{\tau+1}=0$
    \STATE Initialize submodel $\;\;\mathrm{logit} Q_{t,\epsilon_t} = \mathrm{logit} Q_t + \epsilon_t$, and loss function 
    $$L_t(\epsilon_t) = \left(\prod_{s=1}^t\hat{w}_s(A_s,H_s)\right) \mathcal{L}_{bce}(Q_{t,\epsilon_t}(A_t, H_t), Q_{t+1,\epsilon_{t+1}}(a_{t+1},H_{t+1})).$$
    \FOR{$t=\tau$ to $1$}
        \STATE $\epsilon_t \leftarrow \underset{\epsilon_t}{\arg\min} \mathbb{P}_n L_t(\epsilon_t) + \lambda|\epsilon_t|$
    \ENDFOR
\end{algorithmic}
\end{algorithm}

\begin{table*}[t]
\centering
\caption{Absolute bias (mean $\pm$ std) and RMSE of CAPO estimates at horizon $\tau = 10$.}
\vspace{-10pt}
\label{tab:results-tau10}
\resizebox{\linewidth}{!}{%
\begin{tabular}{l
 cccc cccc
 cccc cccc}
\toprule
& \multicolumn{8}{c}{\textbf{Limited Time-varying Confounding}} 
& \multicolumn{8}{c}{\textbf{Expanded Time-varying Confounding}} \\
\cmidrule(lr){2-9} \cmidrule(lr){10-17}

\textbf{Model}
& \multicolumn{4}{c}{\textbf{Bias}}
& \multicolumn{4}{c}{\textbf{RMSE}}
& \multicolumn{4}{c}{\textbf{Bias}}
& \multicolumn{4}{c}{\textbf{RMSE}} \\

& CF 1 & CF 2 & CF 3 & CF 4
& CF 1 & CF 2 & CF 3 & CF 4
& CF 1 & CF 2 & CF 3 & CF 4
& CF 1 & CF 2 & CF 3 & CF 4 \\
\midrule

\rowcolor{gray!6}
G-comp (sl.)
& 3.03$\pm$1.19 & 3.81$\pm$1.52 & 2.81$\pm$1.47 & 4.33$\pm$1.39
& 3.25 & 4.09 & 3.16 & 4.53
& 1.93$\pm$0.80 & 3.34$\pm$0.92 & 1.82$\pm$0.48 & 3.10$\pm$1.21
& 2.08 & 3.46 & 1.88 & 3.32 \\

LTM. (sl.)
& 2.80$\pm$0.89 & 0.63$\pm$0.41 & \textbf{0.26$\pm$0.30} & 3.65$\pm$1.39
& 2.93 & 0.75 & 0.39 & 3.89
& 2.01$\pm$0.71 & 0.58$\pm$0.67 & 0.21$\pm$0.19 & 3.41$\pm$0.87
& 2.14 & 0.87 & 0.28 & 3.51 \\

\rowcolor{gray!6}
D.ACE
& 0.93$\pm$0.16 & 0.52$\pm$0.40 & 0.84$\pm$0.52 & 0.78$\pm$0.26
& 0.95 & 0.65 & 0.98 & 0.82
& 0.98$\pm$0.11 & 0.69$\pm$0.37 & 0.78$\pm$0.46 & 0.87$\pm$0.11
& 0.99 & 0.77 & 0.90 & 0.88 \\

D.LTMLE
& 0.72$\pm$0.56 & 0.59$\pm$0.49 & 0.61$\pm$0.34 & 1.35$\pm$1.11
& 0.91 & 0.76 & 0.70 & 1.73
& 0.51$\pm$0.45 & 0.37$\pm$0.42 & \underline{0.41$\pm$0.26} & 0.78$\pm$0.52
& 0.68 & 0.55 & 0.48 & 0.93 \\

\rowcolor{gray!6}
IGC-Net
& \underline{0.19$\pm$0.14} & \textbf{0.14$\pm$0.10} & \underline{0.36$\pm$0.19} & \underline{0.55$\pm$0.13}
& \underline{0.23} & \textbf{0.17} & \textbf{0.41} & \underline{0.56}
& \underline{0.19$\pm$0.14} & \underline{0.19$\pm$0.09} & \textbf{0.36$\pm$0.21} & \textbf{0.48$\pm$0.09}
& \underline{0.24} & \textbf{0.21} & \textbf{0.42} & \textbf{0.49} \\

$D^3$-Net
& \textbf{0.11$\pm$0.12} & \underline{0.19$\pm$0.27} & 0.51$\pm$0.10 & \textbf{0.39$\pm$0.21}
& \textbf{0.16} & \underline{0.32} & \underline{0.52} & \textbf{0.45}
& \textbf{0.14$\pm$0.09} & \underline{0.16$\pm$0.17} & 0.56$\pm$0.10 & \underline{0.61$\pm$0.32}
& \textbf{0.16} & \underline{0.23} & \underline{0.57} & \underline{0.68} \\

\bottomrule
\end{tabular}
}
\vspace{-0.5\baselineskip}
\end{table*}

\begin{table*}[t]
\centering
\caption{Absolute bias and RMSE of CAPO estimates at $\tau = 15$ under expanded confounding with additional covariates.}
\vspace{-10pt}
\label{tab:results-tau15-numz}
\resizebox{\linewidth}{!}{%
\begin{tabular}{l
 cccc cccc
 cccc cccc}
\toprule
& \multicolumn{8}{c}{\textbf{Additional Covariates: numz = 10}} 
& \multicolumn{8}{c}{\textbf{Additional Covariates: numz = 15}} \\
\cmidrule(lr){2-9} \cmidrule(lr){10-17}

\textbf{Model}
& \multicolumn{4}{c}{\textbf{Bias}}
& \multicolumn{4}{c}{\textbf{RMSE}}
& \multicolumn{4}{c}{\textbf{Bias}}
& \multicolumn{4}{c}{\textbf{RMSE}} \\

& CF 1 & CF 2 & CF 3 & CF 4
& CF 1 & CF 2 & CF 3 & CF 4
& CF 1 & CF 2 & CF 3 & CF 4
& CF 1 & CF 2 & CF 3 & CF 4 \\
\midrule

\rowcolor{gray!6}
D.ACE
& 0.91$\pm$0.16 & 0.40$\pm$0.09 & 0.69$\pm$0.24 & 0.54$\pm$0.34
& 0.93 & 0.40 & 0.73 & 0.63
& 0.96$\pm$0.18 & 0.36$\pm$0.09 & 0.74$\pm$0.27 & \underline{0.29$\pm$0.20}
& 0.97 & 0.37 & 0.78 & \underline{0.35} \\

D.LTMLE
& 0.75$\pm$0.88 & 1.00$\pm$0.95 & 0.55$\pm$0.24 & 0.43$\pm$0.43
& 1.14 & 1.36 & 0.60 & 0.60
& 1.03$\pm$1.04 & 0.92$\pm$0.71 & 0.55$\pm$0.34 & 0.32$\pm$0.26
& 1.45 & 1.15 & 0.64 & 0.40 \\

\rowcolor{gray!6}
IGC-Net
& \underline{0.50$\pm$0.22} & \textbf{0.15$\pm$0.11} & \underline{0.43$\pm$0.31} & \underline{0.39$\pm$0.18}
& \underline{0.55} & \textbf{0.18} & \underline{0.53} & \underline{0.43}
& \underline{0.47$\pm$0.33} & \textbf{0.21$\pm$0.15} & \underline{0.33$\pm$0.23} & 0.38$\pm$0.26
& \underline{0.57} & \textbf{0.25} & \underline{0.40} & 0.46 \\

$D^3$-Net
& \textbf{0.13$\pm$0.13} & \underline{0.19$\pm$0.12} & \textbf{0.30$\pm$0.15} & \textbf{0.14$\pm$0.10}
& \textbf{0.18} & \underline{0.22} & \textbf{0.33} & \textbf{0.17}
& \textbf{0.11$\pm$0.08} & \underline{0.23$\pm$0.22} & \textbf{0.32$\pm$0.18} & \textbf{0.18$\pm$0.11}
& \textbf{0.13} & \underline{0.32} & \textbf{0.37} & \textbf{0.21} \\

\bottomrule
\end{tabular}
}
\vspace{-0.5\baselineskip}
\end{table*}

\section{Implementation and Hyperparameter}
\label{appx:implementation}

Super-Learner-based G-computation and LTMLE are implemented using the \texttt{ltmle} R package, which constructs an optimally weighted ensemble of base learners via 3-fold cross-validation. Candidate learners include a generalized linear model, random forest, XGBoost, and a generalized additive model with regression splines. DeepACE and IGC-Net are implemented using the authors' original codebases~\cite{frauen2023estimating,hess2024igc}, while DeepLTMLE is implemented following the original paper. All deep learning models are trained for 100 epochs. For $D^3$-Net, the Polyak averaging coefficient is fixed at $\beta=0.02$.

Hyperparameters are tuned at $\tau=15$ using 800 trajectories for training and 200 for validation. Models are retrained on all 1000 trajectories using the selected hyperparameters. For $\tau=10$ and $\tau=20$, we reuse the hyperparameters selected at $\tau=15$. We perform random search and select hyperparameters by minimizing the factual loss
$
L_{mse}(Q(A_\tau,H_\tau),Y)+\sum_{t=1}^{\tau}L_{bce}(G(H_t),A_t).
$
For all deep learning methods, we tune the batch size $\{128,256\}$, learning rate $\{5\times10^{-4},10^{-3},5\times10^{-3}\}$, hidden size $\{8,16,32\}$, and dropout rate ${0,0.1}$. Transformer-based models (DeepLTMLE, IGC-Net, and $D^3$-Net) additionally tune the number of layers $\{1,2,3\}$ and attention heads $\{2,4\}$. DeepLTMLE and $D^3$-Net further tune $\alpha\in\{0.05,0.1\}$, while IGC-Net additionally tunes the hidden representation size $\{8,16,32\}$.

\section{Semi-Synthetic Data DGP}\label{appx:dgp}
\paragraph{DGP with Limited Time-varying Confounding}
We adopt the semi-synthetic DGP from DeepACE. Using the preprocessing pipeline of ~\citet{Wang_2020}, we extract 10 time-varying covariates $X_t \in \mathbb{R}^{10}$ over $\tau = 15$ time steps. These covariates are treated as observed patient states and are directly taken from real-world measurements.

Given the observed history, binary treatments $A_t \in \{0,1\}$ are simulated according to a stochastic, dynamic treatment policy. The treatment assignment probability at time $t$ depends on past covariates, past intermediate outcomes, and a treatment intensity variable $\ell_t = \ell_{t-1} + 2(A_t - 1)\bar{X}_t \tanh(Y_t)$, with initialization $\ell_0 = T/2 - 3$. Treatment is assigned as $\pi_t = $
\begin{align*}
\mathds{1}\Bigg\{\sigma\Big( \sum_{i=1}^{h} \frac{(-1)^i}{1 - i}\left(\bar{X}_{t-i} + \frac{\tanh(Y_{t-i})}{2}\right)
- \tanh\!\left(\ell_{t-1} - \frac{T}{2}\right) + \epsilon_t^A \Big) > 0.5 \Bigg\},
\end{align*}
and outcomes are generated according to
\begin{equation*}
Y_{t+1}
= 5 \sum_{i=1}^{h} \frac{(-1)^i}{1 - i}
\tanh\!\left(
\sin(\bar{X}_{1:5,t-i} A_{t-i})
+ \cos(\bar{X}_{5:10,t-i} A_{t-i})
\right)
+ \epsilon_t^Y,
\end{equation*}
where $\bar{X}_{1:5,t-i}$ represent the mean of the first five covariates at step $t-i$, similarly for $\bar{X}_{5:10,t-i}$. The time lag $h=8$. Noise terms satisfy $\epsilon_t^A, \epsilon_t^Y \sim \mathcal{N}(0,0.5^2)$. 
Since our target estimand is the terminal outcome $Y_{\tau}$, intermediate outcomes $Y_t$ for $t<\tau$ are absorbed into $L_{t+1}$, yielding $L_t = (X_t, Y_{t-1})^\top \in \mathbb{R}^{11}$.

\paragraph{DGP with Expanded Time-varying Confounding\;}
In the simple DGP above, patient covariates are unaffected by treatment. To introduce a stronger and more realistic treatment–confounder feedback loop, we further synthesize treatment-dependent time-varying covariates.
Specifically, we generate latent covariates $Z_t \in \mathbb{R}^{d_z}$ with initialization $Z_1 \sim \mathcal{N}(0, I)$ and dynamics
\begin{equation*}
Z_{t+1}
= \omega_1 \cdot Z_t
+ \omega_2 \cdot A_t  \sigma(Z_t^2)
+ \omega_3 \cdot 0.25 \tanh(\bar{X}_t)
+ \epsilon_t^Z,
\end{equation*}
where $\epsilon_t^Z \sim \mathcal{N}(0,0.3^2)$ and the coefficients $\omega = (\omega_1,\omega_2,\omega_3)$ are randomly drawn and fixed to $(0.37,0.42,0.29)$. These synthesized covariates both affect and are affected by treatment. 

Unless otherwise specified, we set ($d_z=5$). To evaluate robustness in higher-dimensional settings, we additionally consider ($d_z=10$) and ($d_z=15$), generated using the same structural equations. We then concatenate $Z_t$ with the observed covariates to form the augmented state
$
X_t^\dagger = (X_t, Z_t)^\top \in \mathbb{R}^{10 + d_z}.
$. Finally, treatment assignment and outcome generation follow the same structural equations as in the Limited Confounding DGP, with $X_t$ replaced by $X_t^\dagger$. Noise distributions remain unchanged.

\end{document}